\documentclass[conference]{IEEEtran}
\IEEEoverridecommandlockouts
\usepackage{cite}
\usepackage{xcolor}
\usepackage{amsmath,amssymb,amsfonts}
\usepackage{amsthm}
\usepackage{mathtools}
\usepackage{algorithm}
\usepackage{algpseudocode}
\usepackage{amsmath}        
\usepackage{amssymb}        
\usepackage{graphicx}
\usepackage{textcomp}
\usepackage{xcolor}
\usepackage{graphicx}
\usepackage{booktabs}
\usepackage{subcaption}
\usepackage{enumitem}
\usepackage{multirow}
\usepackage{tabularx}
\def\BibTeX{{\rm B\kern-.05em{\sc i\kern-.025em b}\kern-.08em
    T\kern-.1667em\lower.7ex\hbox{E}\kern-.125emX}}

\usepackage{tikz}

\usepackage[utf8]{inputenc} 
\usepackage[T1]{fontenc}    
\usepackage{hyperref}       
\usepackage{url}            
\usepackage{booktabs}       
\usepackage{amsfonts}       
\usepackage{nicefrac}       
\usepackage{microtype}      
\usepackage{xcolor}         
\usepackage{adjustbox}
\usepackage{subcaption}
\usepackage{enumitem}

\usepackage{hyperref}
\usepackage{algorithm}
\usepackage{wrapfig}

\setlist[itemize]{noitemsep, topsep=2pt, parsep=1pt, partopsep=0pt}
\setlist[enumerate]{noitemsep, topsep=2pt, parsep=1pt, partopsep=0pt}

\usepackage{etoolbox}
\patchcmd{\thebibliography}{\everypar}{\everypar\setlength{\itemsep}{1pt plus 1pt}\small}{}{}

\theoremstyle{plain}
\newtheorem{theorem}{Theorem}[section]

\newtheorem{lemma}[theorem]{Lemma}

\theoremstyle{definition}

\theoremstyle{remark}

\newcommand{\revision}[1]{\textcolor{black}{#1}}

\newcommand{\sysname}{SNI-GNN}

\begin{document}

\title{SNI-GNN: SmartNIC-Assisted Full-Graph GNN Training with In-Network Embedding Prediction
}



\author{
\IEEEauthorblockN{
Guofan Yu\IEEEauthorrefmark{1},
Sitian Chen\IEEEauthorrefmark{1},
Zhenheng Tang\IEEEauthorrefmark{2},
Xiaowen Chu\IEEEauthorrefmark{3},
Amelie Chi Zhou\IEEEauthorrefmark{1}}
\IEEEauthorblockA{\IEEEauthorrefmark{1}Hong Kong Baptist University}
\IEEEauthorblockA{\IEEEauthorrefmark{2}Hong Kong University of Science and Technology}
\IEEEauthorblockA{\IEEEauthorrefmark{3}Hong Kong University of Science and Technology (Guangzhou)}
}
\maketitle

\begin{abstract}
Full-graph GNN training delivers high accuracy but scales poorly on multi-server clusters due to heavy, irregular inter-node embedding exchanges. We present SNI-GNN, a SmartNIC-assisted full-graph training system that reduces communication while preserving accuracy by predicting remote embeddings in-network. SNI-GNN deploys a lightweight linear-trend predictor on SmartNICs to refine cached historical embeddings, coupled with an importance-based boundary-node sampling policy and an asynchronous DPU--GPU data pipeline with intermediate-result reuse. We provide error and convergence bounds showing that predictor bias remains controlled under bounded second-order dynamics and yields standard non-convex convergence with inexact gradients. Implemented on NVIDIA BlueField-3, SNI-GNN integrates with state-of-the-art full-graph systems, cuts communication by 21--45\%, achieves 1.3--3.6$\times$ end-to-end speedups over BNS-GCN and up to 1.29$\times$ over baseline SANCUS, with accuracy loss $\leq 0.01$, and scales efficiently to 16 GPUs on graphs with up to tens of millions of edges. These results indicate SmartNIC-based in-network prediction is a practical complement to partitioning and compression techniques for communication-efficient full-graph GNN training at scale.
\end{abstract}


\section{Introduction}\label{sec:intro}



%
Graph Neural Networks (GNNs) are a class of deep learning models specifically designed for graph-structured data. They have achieved remarkable success across a wide range of applications, including node classification~\cite{li2023early, liu2020towards}, link prediction~\cite{zhang2018link, cai2020multi}, and graph classification~\cite{xu2018powerful, gao2021topology}. Owing to their strong performance in capturing topological patterns and relational structures, GNNs have become a cornerstone technique in numerous domains such as social networks~\cite{liu2021content, qiu2018deepinf}, recommendation systems~\cite{li2023survey}, fraud detection~\cite{wang2019neural, dou2020enhancing} and molecular biology~\cite{cai2022fp}.

A prominent training paradigm in GNNs is \textbf{full-graph} training, which processes the entire graph in each training iteration. Compared to sampling-based approaches, full-graph training enables precise aggregation of neighborhood information without introducing sampling noise~\cite{lin2023comprehensive, bajaj2024graph}. This results in faster convergence and superior accuracy, particularly in applications that require high structural fidelity, such as knowledge graph reasoning~\cite{du2021cogkr} and molecular property prediction~\cite{panapitiya2024fragnet}.

However, full-graph training also poses significant scalability challenges, particularly in distributed settings. Since full-graph training requires access to the entire graph in each iteration, it often exceeds the memory capacity of a single GPU or even a single machine, especially for large-scale real-world graphs containing millions or billions of nodes and edges~\cite{zheng2020distdgl, park2022ginex, mostafa2022sequential}. As a result, the graph and its associated features must be partitioned and distributed across multiple computing nodes.

This distributed setup introduces substantial communication overhead. During training, each node must frequently exchange feature and embedding information with other nodes to obtain the latest states of neighboring vertices. The inherent sparsity and irregularity of graph data further exacerbate the challenge, as they lead to unpredictable and imbalanced communication patterns in terms of volume, frequency, and destinations. Consequently, communication often becomes the dominant bottleneck, significantly limiting training throughput and hindering system scalability.

To mitigate communication overhead, recent studies have explored a range of techniques such as overlapping communication and computation~\cite{wang2021flexgraph,wan2023scalable,zheng2022distributed}, using stale embeddings~\cite{wan2022pipegcn,peng2022sancus}, and compressing data via quantization~\cite{wan2023adaptive}. Despite these efforts, most existing works focus on single-machine multi-GPU setups, largely constrained by the scale of publicly available datasets. In contrast, industrial-scale graphs demand multi-server training architectures. Yet, most full-graph training systems supporting multi-server configurations are simple extensions of single-node designs~\cite{peng2022sancus,zhang2024sylvie,wan2022pipegcn}, failing to exploit opportunities for communication optimization in multi-node environments. In comparison, mini-batch GNN training systems have begun to incorporate inter/intra-node communication-aware strategies, such as adaptive sampling~\cite{song2024granndis}.

In modern data centers, \textbf{SmartNICs} are increasingly adopted to address communication challenges in large-scale distributed systems. Unlike traditional NICs, SmartNICs offer programmable computing capabilities, enabling them to offload certain network and system-level tasks from CPUs. This architectural shift enhances overall system efficiency and reduces network congestion~\cite{guo2024software, jin2025os2g}.

We identify a new opportunity from the vantage point of SmartNICs. Positioned in the inter-node data path and equipped with programmable compute and memory, SmartNICs can cache and process embeddings in transit without consuming host CPU/GPU cycles. Empirically, we observe that node embeddings evolve smoothly over time during training, making them amenable to lightweight temporal forecasting. This motivates in-network prediction: refine stale embeddings directly on the SmartNIC, reduce cross-node exchange frequency, and control staleness without introducing new synchronization.

We present \textbf{\sysname{}}, a SmartNIC-assisted and communication-efficient full-graph GNN training system for multi-server clusters. \revision{In contrast to software-only optimizations that contend for host-side resources, SNI-GNN leverages the DPU to perform active staleness mitigation directly on the data path. By offloading a lightweight linear-trend predictor to the SmartNIC, our system decouples embedding extrapolation from the bottlenecked host CPU and eliminates the redundant PCIe traffic associated with host-based prediction. This design ensures strict hardware isolation between the prediction logic and the main training process.} 

\revision{To manage the DPU’s specialized resource constraints, \sysname{} employs boundary-node importance sampling to prioritize the most influential neighbors. It further utilizes an asynchronous DPU–GPU pipeline with intermediate-result reuse to overlap in-network prediction, data transfer, and GPU computation. Based on the assumption of bounded second-order dynamics where embeddings evolve smoothly with bounded curvature, we provide a rigorous theoretical foundation. This analysis proves that the predictor’s bias remains controlled and establishes standard non-convex convergence despite the use of inexact gradients.}


We prototype SNI-GNN on NVIDIA BlueField-3 DPUs and integrate it with state-of-the-art full-graph systems, including SANCUS and NeutronTP. Across datasets with up to tens of millions of nodes and scaling up to 16 GPUs, SNI-GNN reduces inter-node communication volume by 21–45\%, delivering 1.3–3.6$\times$ end-to-end speedups over the sampling-based BNS-GCN~\cite{Wan2022BNSGCNEF}. \revision{Notably, it achieves up to a 1.29$\times$ marginal speedup over an optimized SANCUS~\cite{peng2022sancus} baseline by eliminating PCIe bottlenecks and host-resource contention through in-network prediction. Our evaluation across diverse architectures, including GraphSAGE and GAT, confirms the framework's generalizability, while experiments with 4-layer and 6-layer GNNs demonstrate its robustness against staleness amplification in deeper networks. Throughout these stress tests, SNI-GNN consistently maintains model accuracy within 1\% of the ideal synchronous training.
}

\revision{\sysname{} makes the following contributions:
\begin{itemize}[leftmargin=*] 
\item \textbf{Active Mitigation Architecture:} We propose a SmartNIC-assisted training stack that moves beyond passive staleness tolerance to {active in-network mitigation}. Unlike prior host-centric systems, \sysname{} targets inter-node bottlenecks by isolating prediction logic on the DPU, eliminating host-resource contention and PCIe-induced synchronization stalls. 
\item \textbf{Hardware-Aware Predictor:} A lightweight linear-trend predictor specifically optimized for SmartNIC constraints. By combining {ring-buffer histories} with {boundary-node importance sampling}, we enable high-fidelity embedding estimation for billion-scale graphs within the limited memory and compute budget of the DPU's ARM cores. 
\item \textbf{Conflict-Free Asynchronous Pipeline:} A specialized DPU–GPU and inter-node execution pipeline featuring {intermediate-result reuse}. This design enables the complete overlapping of prediction, RDMA transmission, and GPU computation, effectively hiding communication latency without introducing new synchronization barriers. 
\item \textbf{Theoretical Rigor and System Validation:} We provide error and convergence guarantees under bounded second-order dynamics. Our BlueField-3 prototype demonstrates that \sysname{} is not a standalone extension but a {generalizable enhancement} that integrates with existing full-graph systems, yielding up to {1.29$\times$ marginal speedups} over already-optimized baselines. 
\end{itemize}
}

\section{Background and Motivation}
\label{sec:Background}
\subsection{Graph Neural Networks}

Graph Neural Networks (GNNs) are deep learning architectures designed to process graph-structured data, where nodes represent entities (e.g., users, molecules) and edges encode relationships. Unlike traditional neural networks for grid-like data, GNNs leverage message-passing mechanisms~\cite{gilmer2017neural} to iteratively refine node representations by propagating and aggregating features across the graph’s topology. 

A GNN model comprises $L$ layers, where each layer transforms node embeddings to capture increasingly complex structural patterns. Let \(h_v^{(l)}\) denote the feature vector of node $v$ at $l$-th layer, which is updated through two key steps: 
\begin{itemize}[leftmargin=*]
    \item \textbf{Aggregation:} Node $v$ gathers feature vectors from its neighbors $\mathcal{N}(v)$, applying a permutation-invariant aggregation function (e.g., sum, mean) to handle variable-sized inputs: 
    \begin{equation}\label{eq:aggregate}
        a_v^{(l)}=\text{AGGREGATE}^{(l)} \left( \{ h_u^{(l-1)} \mid u \in \mathcal{N}(v) \}\right) 
    \end{equation}
    \item \textbf{Update:} The aggregated features $a_v^{(l)}$ are combined with $v$'s previous state $h_v^{(l-1)}$ via a function $f_\theta$ (e.g., MLP), where $\theta$ are learnable parameters:
    \begin{equation}\label{eq:update}
        h_v^{(l)} = f_\theta^{(l)} \left( h_v^{(l-1)},a_v^{(l)}\right) 
    \end{equation}
\end{itemize}

Through successive layers, nodes integrate information from their $L$-hop neighborhoods, enabling hierarchical feature learning across graph topology. 
Despite their expressive power, training GNNs on large-scale graphs presents significant challenges.
The message-passing paradigm requires frequent access to neighbor embeddings, creating {memory bottlenecks} when storing intermediate states for billions of nodes. 
Moreover, neighborhood aggregation triggers irregular data access patterns that strain computational resources~\cite{cai2021dgcl,gandhi2021p3}. 





\subsection{Distributed Full-Graph GNN Training}

\begin{figure}[t]
\centering
\includegraphics[width=0.45\textwidth]{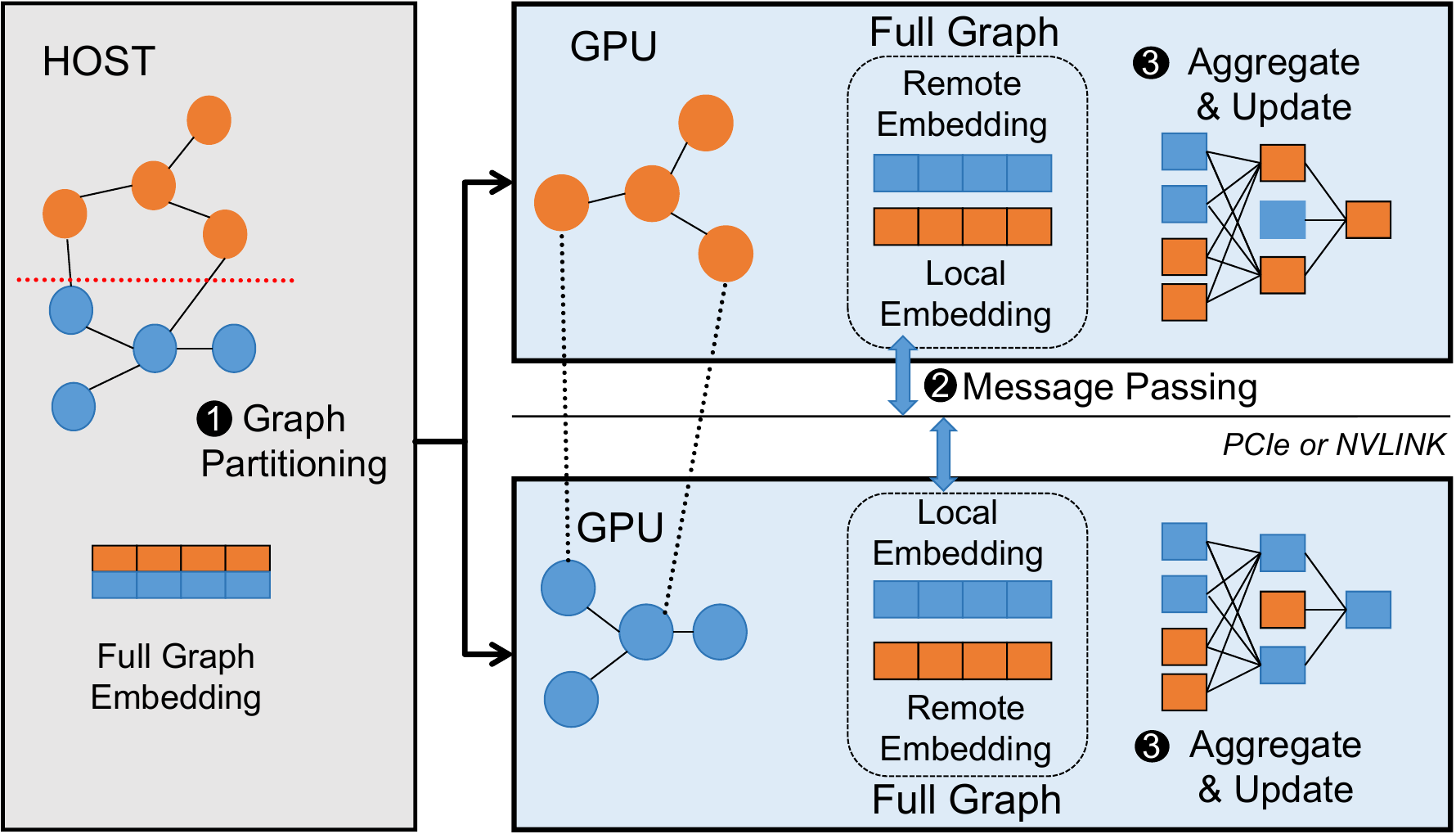}
\caption{General flow of distributed full-graph GNN training}
\label{fig:training}
\end{figure}

Large-scale graph datasets necessitate distributed training frameworks to overcome memory and computational constraints.
Two primary paradigms exist, including {mini-batch training} and {full-graph training}.
\textbf{Mini-batch training} partitions a graph into small batches for sequential processing~\cite{zheng2020distdgl,fey2021gnnautoscale,yang2022gnnlab}. Although this approach introduces lower memory footprint due to neighbor sampling, it can also lead to decreased accuracy due to incomplete neighborhood aggregation.
\textbf{Full-graph training}, on the contrary, processes the entire graph simultaneously across distributed workers~\cite{wan2022pipegcn,NeutronTP,peng2022sancus,Wan2022BNSGCNEF}. This approach preserves full-neighbor aggregation semantics and enables full-batch gradient descent, yielding higher accuracy and faster convergence.
\emph{Our paper focuses on full-graph training for its accuracy benefits, and strive to address its unique scalability challenges.}

Full-graph training performance is severely constrained by communication bottlenecks arising from graph partitioning and cross-worker dependencies. As shown in Figure~\ref{fig:training}, the full-graph GNN training can be decomposed into three stages: 
\begin{itemize}[leftmargin=*]
    \item \emph{Partitioning}: The graph is split across workers using various graph partitioning strategies (such as METIS~\cite{karypis1998fast}, LDG~\cite{stanton2012streaming} or KaHIP~\cite{edgepartitioning2019}). Vertex embeddings of the entire graph are copied onto each worker (hence full-graph training). This stage is performed only once during the entire training. 
    \item \emph{Message Passing}: For each training iteration, workers exchange vertex embeddings to resolve cross‐partition dependencies. For example, if node $v$ on Worker A has a neighbors $u$ on Worker B, Worker A must fetch $h_u^{(l-1)}$ from B to update $h_v^{(l)}$. 
    \item \emph{Aggregate and Update:} Each worker aggregates the received/local neighbor embeddings using Equation~\ref{eq:aggregate} and updates vertex features using Equation~\ref{eq:update}. Parameters of the aggregate and update functions are updated through forward and backward propagation.
\end{itemize}

We conducted experiments on an eight-GPU V100 cluster with varying node configurations: 2 nodes with 4 GPUs each (2N\_4G), 4 nodes with 2 GPUs each (4N\_2G), and 8 nodes with 1 GPU each (8N\_1G). All nodes are interconnected via a 10 Gbps Ethernet. We evaluated full-graph GNN training across two dataset scales: ogbn-products~\cite{2020arXiv200500687H} and Reddit~\cite{hamilton2017inductive}. Detailed configurations can be found in Section~\ref{sec:eval}.
Figure~\ref{fig:motivation time breakdown} presents the time breakdown on computation (aggregation and update) and communication (message passing).
As clearly evidenced in the figure, the communication consistently dominates the training overhead across all configurations and datasets. 

\begin{figure}[t]
\centering
\includegraphics[width=0.45\textwidth]{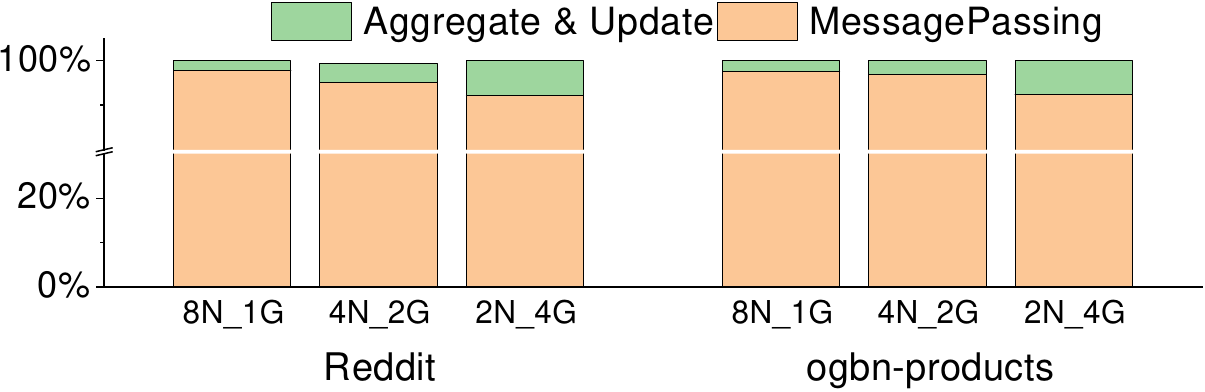}
\caption{Full-Graph GNN training time breakdown}
\label{fig:motivation time breakdown}
\end{figure}

\subsection{Historical Embeddings and Staleness in GNN training}

To mitigate the large communication costs, a common technique is to reuse \textbf{historical embeddings} from previous iterations~\cite{yu2022graphfm, peng2022sancus, huang2023freshgnn, fey2021gnnautoscale}. 
Instead of fetching the latest embedding $h_u^{(l-1, t)}$ for a remote neighbor $u$, a worker uses a cached version $\tilde{h}_u^{(l-1,t)}=h_u^{(l-1, t-1)}$.

While this approach reduces network pressure, it introduces \textbf{embedding staleness} due to the discrepancy between the cached historical value and the current, updated embedding. 
As training progresses, stale embeddings can distort the aggregation step, potentially degrading model convergence and final accuracy~\cite{yu2022graphfm}. This is especially true for deep GNN architectures or graphs with dynamic topological dependencies.
According to Equation~\ref{eq:aggregate} and \ref{eq:update}, using historical embeddings during training leads to an approximation of message aggregation, as formulated below:

\begin{equation}
h_v^{(l, t)} \approx f_\theta\left(h_v^{(l-1, t-1)}, \text{AGGREGATE}\left(\{\tilde{h}_u^{(l-1,t)} \mid u \in \mathcal{N}(v)\}\right)\right)
\end{equation}
where the historical embedding \(\tilde{h}_u^{(l-1,t)}= h_u^{(l-1, t-1)}\). More generally, we could use \(\tilde{h}_u^{(l-1,t)}= h_u^{(l-1, t-i)}, i\le t-1\), meaning that we could use older historical embeddings to approximate the embedding in the current iteration and $i$ is a parameter that trades-off communication and accuracy.







 To quantify the staleness introduced by historical embeddings, 
 we conducted a comparative experiment on four nodes using Reddit and ogbn-products datasets.
 We vary the synchronization interval from \revision{1 to 3}, meaning the fresh embeddings will be communicated every 2-4 iterations. When there's no fresh embedding, the training will use the most up-to-date historical embedding.
 To ensure convergence, the stale synchronization starts at 30-th epoch.
 The staleness is measured using the $L2$ norm of the difference between historical embeddings and fresh embeddings, namely $|| h_u^{(l - 1, t)} - \tilde{h}_u^{(l - 1,t)} ||$.
 The results in Figure~\ref{fig:staleness comparison} quantitatively validate the staleness problem inherent in historical embeddings. Across datasets, we observe a clear positive correlation between communication synchronization interval and embedding staleness and the staleness persists throughout the training process. Larger staleness usually means poorer accuracy. Therefore, if we aim to reduce communication frequency to mitigate performance bottleneck while preserving accuracy, we need to \textbf{bridge the staleness gap} through intelligent local computation that brings historical embeddings closer to their current values.


\begin{figure}[t]
    \centering
    \includegraphics[width=0.4\textwidth]{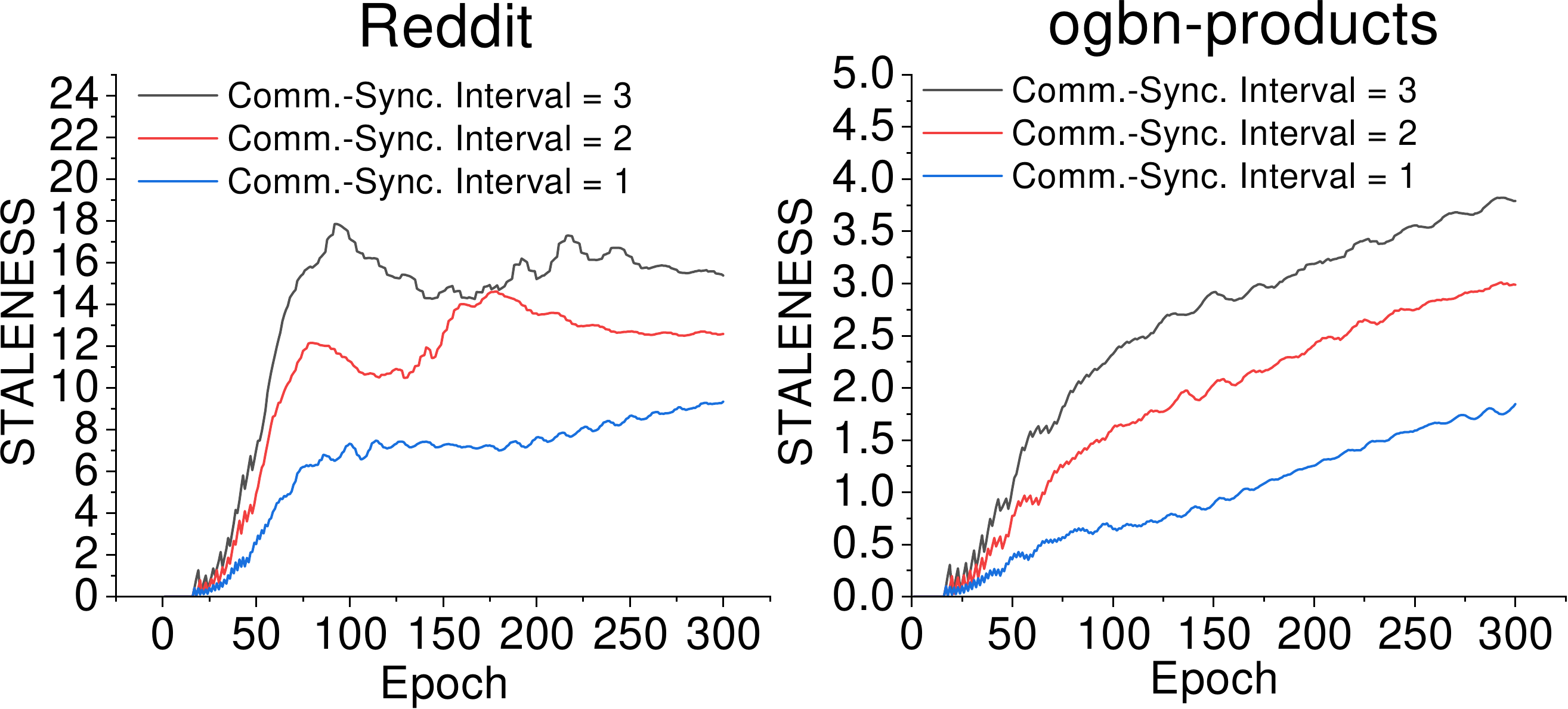}
    \caption{\revision{Comparison of embedding staleness with different communication synchronization intervals.} 
    }
    \label{fig:staleness comparison}
\end{figure}

\subsection{The Opportunity of SmartNICs}
Smart Network Interface Cards (SmartNICs) are programmable hardware components integrated into the network data path. Unlike traditional NICs, they possess onboard multi-core processors, memory, and accelerators, allowing them to offload and execute operations directly on data in transit~\cite{wong2025gpus}.


In distributed machine learning, SmartNICs have been explored to accelerate collective communication operations (e.g., broadcast, all-reduce)~\cite{9680428, khalilov2024network}, compress gradient data~\cite{rebai2024squeezenic, gu2024omniccl}, and offload optimizer states~\cite{sun2024luwu}. Their strategic position is key: sitting between the host and the network, they can intercept, process, and cache data without consuming GPU or CPU resources. This makes them ideal for implementing in-network caching and computation to improve scalability and resource utilization in large-scale ML clusters.
However, the application of SmartNICs to GNN training remains largely unexplored. The unique challenges of GNNs, including irregular communication, high-dimensional embeddings, and the staleness problem, align perfectly with the capabilities of SmartNICs. 
This work takes the first step toward filling this gap by designing a SmartNIC-assisted full-graph GNN training framework.

\begin{figure}[t]
\centering
\includegraphics[width=0.25\textwidth]{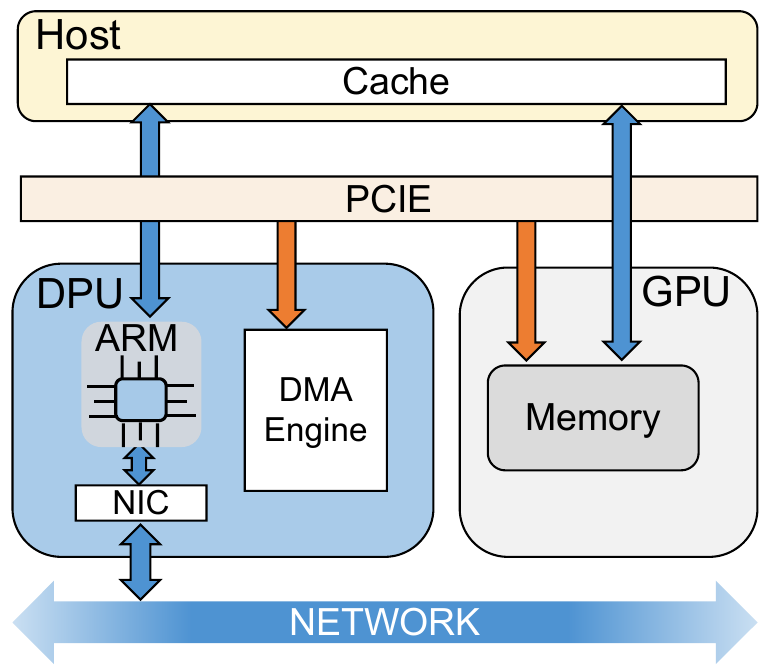}
\caption{System architecture with BlueField-3 SmartNIC.}
\label{fig:System with SmartNIC}
\end{figure}

\section{Design Rationale}\label{sec:rationale}

As shown in Figure~\ref{fig:motivation time breakdown}, communication is the main performance bottleneck in full-graph GNN training.
A straightforward and well-studied approach to reduce communication overhead is to utilize historical embeddings stored locally.
While historical embeddings effectively decrease communication volume, their accumulating drift inevitably degrades model accuracy. This necessitates a \emph{compensation mechanism} that can proactively refresh stale embeddings.
However, implementing such mechanism naively on already-bottlenecked GPUs would simply replace communication overhead with computational overhead.

\revision{While host CPUs might appear to have idle cycles, they are heavily taxed in high-performance GNN training with graph partitioning, data loading, and orchestrating GPU kernels. Offloading the compensation mechanism to the CPU would necessitate bidirectional PCIe traffic—fetching cached embeddings from the NIC, processing them in the CPU, and pushing predictions back to the GPU, thereby congesting the already busy PCIe channels used for CPU-GPU data exchange.}

SmartNICs, on the other hand, offer a strategic solution to this impasse. Positioned directly within the communication path and equipped with dedicated compute and memory resources, they can perform lightweight embedding computations without contending for GPU cycles. This architectural advantage enables \sysname{} to address the staleness problem at its source (i.e., in the network) while keeping the training process undisturbed. 
However, effectively leveraging SmartNICs for this purpose requires addressing a key design question: \emph{How to create a staleness compensation mechanism that is {\bf accurate} enough to mitigate staleness while {\bf lightweight} enough to fit in SmartNICs' limited computation power? }

To answer this question, we studied the temporal evolution of embeddings during GNN training. Figure~\ref{fig:reddit_feature} shows the trajectory of individual embedding dimensions across three representative nodes, revealing a crucial pattern: \textbf{hidden features evolve smoothly and gradually over time rather than exhibiting erratic, unpredictable changes. }
This observed smoothness provides the theoretical foundation for our predictor-based approach: if embedding dimensions follow consistent, low-variance trajectories, then their near-future values can be reliably estimated from recent history. Consequently, we can design a temporal predictor that extrapolates from cached embeddings to generate accurate, stale-free approximations, achieving effective staleness compensation without exceeding SmartNIC computational constraints.


\begin{figure}[t]
\centering
\includegraphics[width=0.49\textwidth]{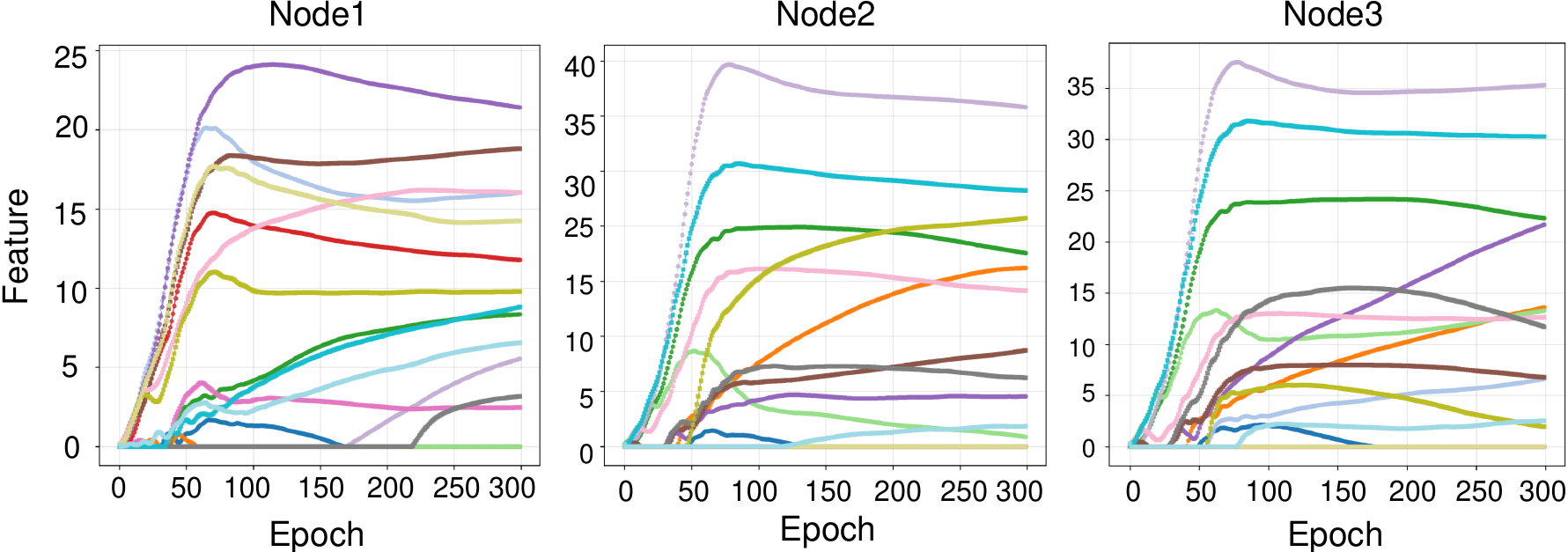}
\caption{Predictable embedding evolution across training epochs. Individual dimension trajectories for three randomly sampled nodes (Reddit dataset) exhibit smooth, gradual changes, validating that lightweight temporal prediction can effectively compensate for embedding staleness.
}
\label{fig:reddit_feature}
\end{figure}

For a prediction model, it is necessary to update the model parameters with the latest embeddings each time. This allows the model to continuously adapt to the changing direction of embeddings as training progresses. 
As shown in Figure~\ref{fig:reddit_feature}, \textbf{different dimensions exhibit distinct curve patterns}, implying that an ideal predictor would require separate, complex models per dimension. This would create a significant computational burden, making our mechanism impractical for resource-constrained SmartNICs. 
To resolve this tension, our design employs two key strategies: 1) a single, highly lightweight linear-trend predictor shared across all dimensions to minimize computational overhead, and 2) an importance-based sampling policy that selectively updates only the most critical nodes and dimensions, ensuring the predictor's workload remains feasible without sacrificing accuracy.

\section{Design Details of \sysname{}}
\label{sec:Design and method}

\begin{figure*}[t]
\centering
\includegraphics[width=0.8\textwidth]{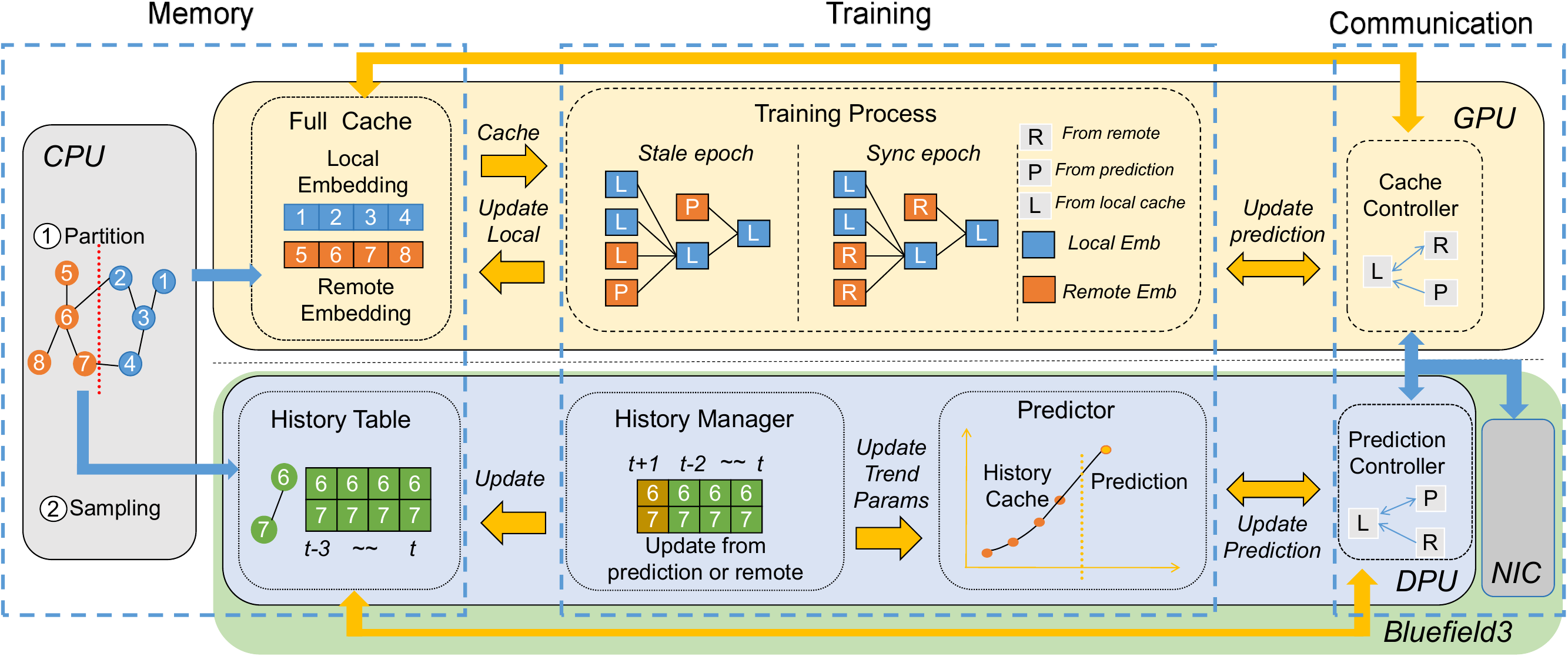}
\caption{System architecture for SmartNIC-assisted full-graph GNN training. Each node has a GPU and a SmartNIC. SmartNICs handle embedding caching, asynchronous data transfer, and in-network prediction to enable efficient decentralized full-graph training.}
\label{fig:System Architecture}
\end{figure*}

\subsection{System Overview}

The core idea of \sysname{} is to mitigate embedding staleness by predicting the current values of remote embeddings directly within the network, thus reducing the need for frequent synchronization. 
Our design is guided by two key observations from empirical studies of GNN training dynamics. First, as shown in Figure~\ref{fig:reddit_feature}, embedding vectors evolve smoothly over time, making them amenable to temporal forecasting. 
Second, different embedding dimensions exhibit distinct evolution patterns, suggesting that per-dimension modeling would be ideal but computationally prohibitive for resource-constrained SmartNICs. 
This tension motivates our design of a single lightweight predictor that operates across all dimensions, complemented by strategic sampling to maintain feasibility.
In the following, we first introduce the design details of our lightweight predictor model (Section~\ref{sec:predictor}), followed by SmartNICs-specific system optimizations to reduce prediction latency (Section~\ref{sec:predictor optimization}). 

{\bf Training Workflow.}
Figure~\ref{fig:System Architecture} illustrates the overall architecture of \sysname{} in a multi-node distributed training setup.
As initialization, \sysname{} performs a one-time graph partitioning and each node holds a subgraph. 
Each node consists of a GPU for model computation and a SmartNIC (DPU) for in-network processing. During training, the DPU maintains historical embedding caches and performs lightweight prediction, while the GPU handles the main GNN computation. The architecture supports two operational modes: Sync Epochs, where fresh embeddings are exchanged between nodes and DPU caches are updated, and Stale Epochs, where DPUs supply predicted embeddings to local GPUs. This design effectively positions the SmartNIC as an intelligent communication filter between the GPU and network, enabling embedding prediction directly in the data path.

\subsection{Lightweight Linear-Trend Predictor}
\label{sec:predictor}

Based on the design rationale, we implement a lightweight linear-trend predictor that aims to reduce the staleness gap $\Delta_v^{(t)} = \| h_v^{(t)} - \tilde{h}_v^{(t)} \|$ without increasing inter-node communication.

The predictor estimates the embedding at the current time step by extrapolating the recent trend relative to the latest stored embedding. The prediction target is given in Equation~\ref{eqa:prediction purpose}.
\begin{equation}
    \mathbf{\hat{h}}_v^{(t)} = \underbrace{\mathbf{\bar{h}}_v}_{\text{baseline}} + \underbrace{\Delta t \cdot \mathbf{\Phi}_v \cdot e^{-\lambda \Delta t}}_{\text{trend}}
    \label{eqa:prediction purpose} 
\end{equation}
where:
\begin{enumerate}
    \item $\Delta t = t - \tau_{\text{latest}}$ is the number of steps since the most recent stored embedding.
    \item $\mathbf{\bar{h}}_v = \alpha\mathbf{h}_v^{(\tau_{\text{latest}})} + (1-\alpha)\overline{\mathcal{H}_v}$ is an exponential-moving-average (EMA) baseline (we set $\alpha=0.3$ in our implementation).
    \item $\mathbf{\Phi}_v$ is a weighted trend vector estimated from recent slope estimates in the history.
    \item $e^{-\lambda \Delta t}$ (we use $\lambda=0.2$) exponentially decays the trend contribution with time.
\end{enumerate}
The detailed prediction method is shown in Algorithm~\ref{alg:Embedding Prediction Algorithm}.

\begin{algorithm}[t]\small
\caption{In-Network Embedding Prediction Algorithm}
\label{alg:Embedding Prediction Algorithm}  
\begin{algorithmic}
\Procedure{Predict}{$\mathcal{H}_v, t$}
\State $\tau_{\text{latest}} \gets \max(\tau_k \in \mathcal{H}_v)$ 
\State $\Delta t \gets t - \tau_{\text{latest}}$
\State $\mathbf{\bar{h}}_v \gets 0.3 \cdot \mathcal{H}_v[\tau_{\text{latest}}] + 0.7 \cdot \text{mean}(\mathcal{H}_v)$
\State $\mathbf{\Phi}_v \gets \mathbf{0}$; $W \gets 0$
\For{$k = 2$ to $K$}
    \State $\delta \tau \gets \tau_k - \tau_{k-1}$
    \State $\mathbf{\phi}^{(k)} \gets (\mathcal{H}_v[k] - \mathcal{H}_v[k-1]) / \delta \tau$
    \State $w_k \gets \exp(-|\tau_{\text{latest}} - \tau_k|/5.0)$
    \State $\mathbf{\Phi}_v \gets \mathbf{\Phi}_v + w_k \cdot \mathbf{\phi}^{(k)}$
    \State $W \gets W + w_k$
\EndFor
\State $\mathbf{\Phi}_v \gets \mathbf{\Phi}_v / W$
\State \Return $\mathbf{\bar{h}}_v + \Delta t \cdot \mathbf{\Phi}_v \cdot \exp(-0.2 \Delta t)$
\EndProcedure
\end{algorithmic}
\end{algorithm}

\revision{To justify our choice, we compare the linear-trend predictor against ARIMA~\cite{box1976analysis}, HOLT~\cite{holt2004forecasting}, and an LSTM~\cite{hochreiter1997long} baseline. As shown in Figure~\ref{fig:prediction_method} and Table~\ref{tab:method_comparison}, while ARIMA achieves slightly higher peak accuracy, its computational cost is nearly an order of magnitude higher than our approach, making it unsuitable for real-time in-network processing.
The comparison between HOLT and Linear-Trend is particularly critical. Although both are lightweight and exhibit similar computation times per epoch, our Linear-Trend predictor achieves $\approx$0.7\% higher peak accuracy (92.06\% vs. 91.36\% on Reddit) and demonstrates more stable convergence with significantly fewer gradient oscillations. From an architectural perspective, HOLT requires maintaining multiple state variables (level, trend, and seasonal components) per embedding dimension, which rapidly exhausts the DPU's limited on-board memory. Our Linear-Trend model minimizes this state storage, enabling the system to support larger graph partitions and a higher sampling ratio within the DPU's resource budget.
}

\begin{figure}[t]
  \centering
  \begin{minipage}[c]{0.51\linewidth}
    \centering
    \includegraphics[width=0.98\linewidth]{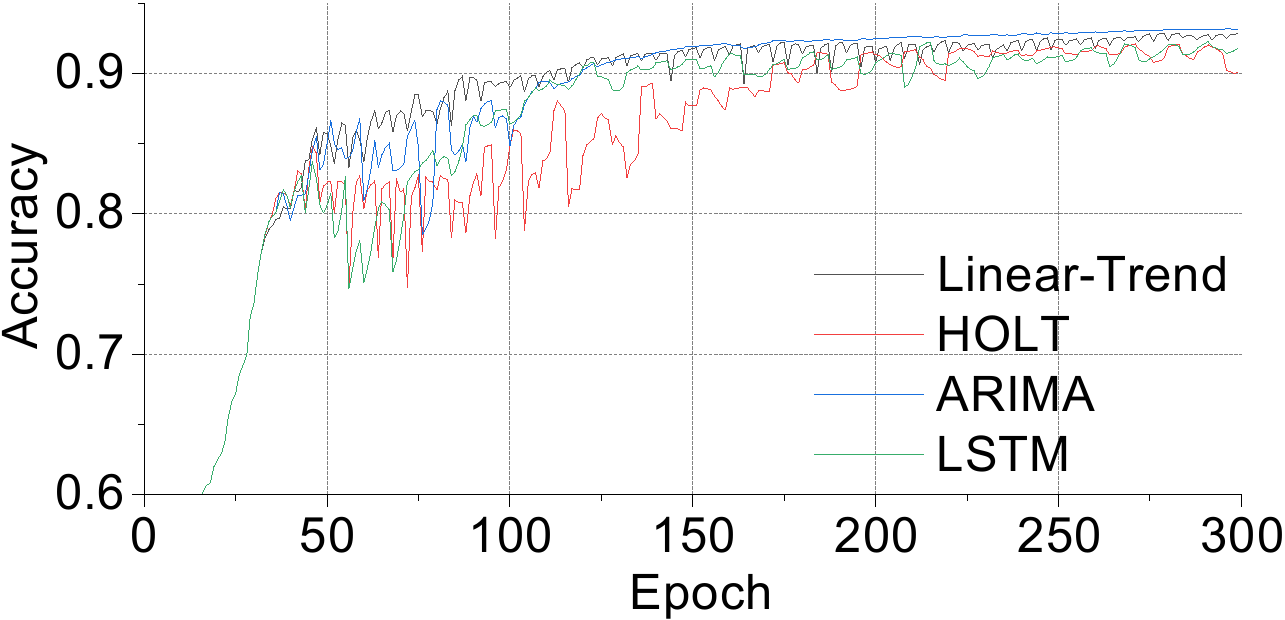} 
    \caption{Comparison of prediction accuracy curves across different forecasting methods}
    \label{fig:prediction_method}
  \end{minipage}%
  \hfill
  \begin{minipage}[c]{0.45\linewidth}\small
    \centering
    \captionof{table}{Computation time and peak accuracy of different predictors}
    \label{tab:method_comparison}
    
    \setlength{\tabcolsep}{3pt} 
    \begin{tabular}{lcc}
      \toprule
      Method       & Time (s) & Acc. (\%) \\ 
      \midrule
      Ours         & 48.2    & 92.88    \\
      HOLT         & 46.83   & 92.14    \\
      ARIMA        & 425.2   & 93.18    \\
      LSTM         & 57.37   & 92.39    \\
      \bottomrule
    \end{tabular}
  \end{minipage}
  \vspace{1pt}
\end{figure}


\subsection{SmartNIC Implementation and Optimizations}
\label{sec:predictor optimization}
The predictor runs on the SmartNIC and maintains a historical embedding table for sampled remote nodes. Training uses a fixed synchronization interval $s$. At each sync epoch every worker broadcasts its latest embeddings; the SmartNIC updates its history from received embeddings and forwards the fresh embeddings to the local GPU. During stale epochs (between syncs) the SmartNIC predicts the next embedding using the linear-trend model, and supplies the predicted embeddings to the GPU for aggregation.

Each SmartNIC maintains a fixed-size circular buffer for each tracked remote node:
\begin{equation}
    \mathcal{H}_v = \left\{ \left( \tau_k, \mathbf{h}_v^{(\tau_k)} \right) \right\}_{k=1}^K \in \mathbb{R}^{K \times D}, \quad \tau_k \in \mathbb{Z}^+
\end{equation}
where $\tau_k$ denotes the global step when embedding $\mathbf{h}_v^{(\tau_k)}$ was received, and $K$ is the buffer capacity . 

However, due to the constraints on computing power and memory of the DPU on the SmartNIC, a significant overhead would be introduced if we re-construct the entire time series for all nodes and perform trend calculation every time we update the historical records during synchronization rounds, as specified in Algorithm 1. Therefore, to further improve the efficiency of the predictor and keep the SmartNIC overhead low, we have implemented several optimizations.

\begin{figure}[t]
\centering
\includegraphics[width=0.34\textwidth]{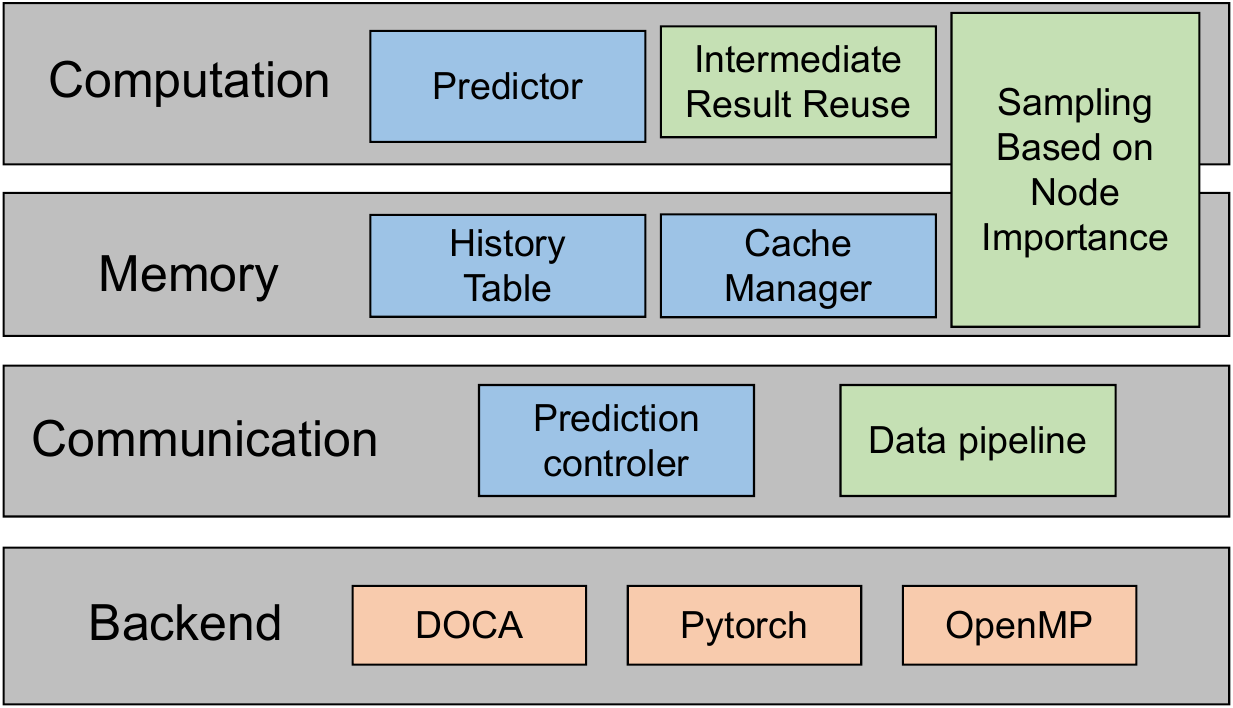}
\caption{Predictor design overview. The blue modules represent the functional modules of the predictor, while the green modules represent the optimization schemes we adopted and their corresponding optimization levels.}
\label{fig:Embedding Predictor overview}
\end{figure}

Figure~\ref{fig:Embedding Predictor overview} shows the predictor architecture. The predictor contains four logical modules: the historical-embedding table, a cache manager, the predictor core, and a prediction controller. Together these modules manage memory, computation, and SmartNIC–GPU communication. We apply three optimizations: importance-based sampling, an asynchronous data pipeline, and intermediate-result reuse to keep SmartNIC overhead low.

%


\subsubsection{Node sampling}
To reduce the memory and computational overhead on the SmartNIC, we implement an importance-based node sampling strategy for the embedding predictor. In full-graph training, each worker requires access to remote node embeddings, which can strain the DPU's limited resources. Our approach prioritizes prediction for the most influential remote nodes while using standard historical embeddings for the rest.

We define a node's importance by the number of edges it has connecting to the local worker's subgraph, a technique known as boundary node sampling. Nodes with more connections are considered more critical to the local subgraph's updates and are selected for prediction. The \textbf{sampling\_ratio} parameter controls the fraction of high-importance nodes that will receive predicted embeddings, allowing a trade-off between prediction accuracy and computational load.

\begin{figure}[t]
\centering
\includegraphics[width=0.4\textwidth]{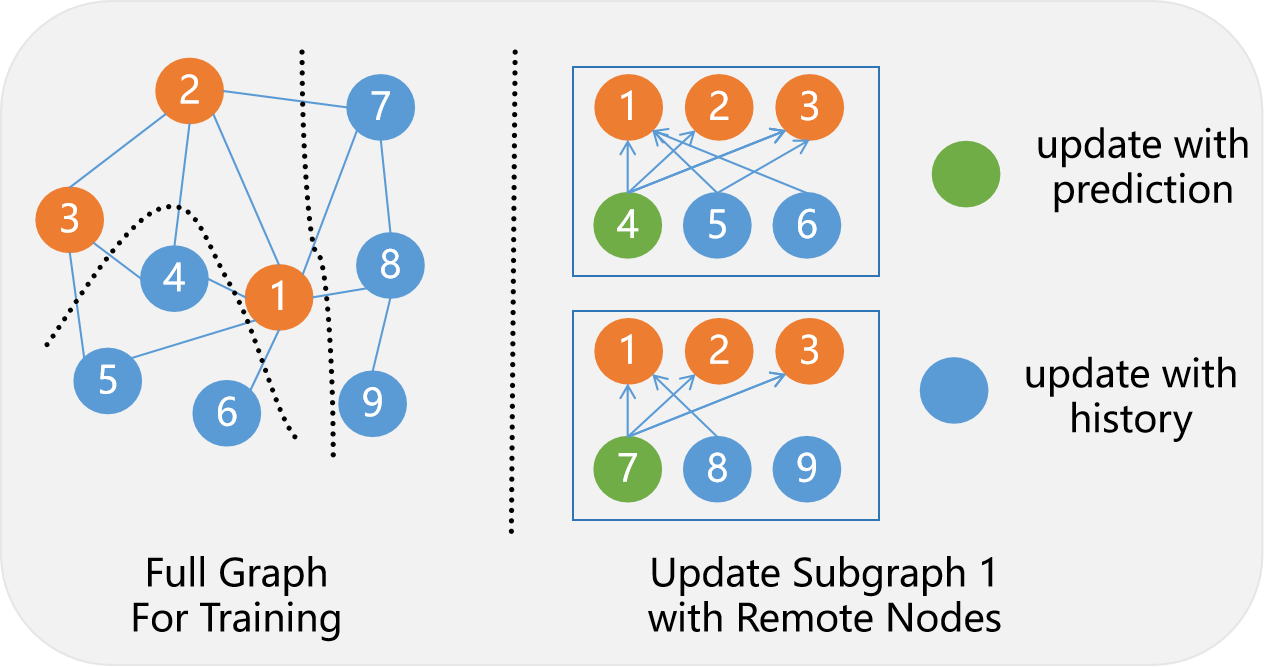}
\caption{Node sampling methods (taking the full-graph training of a three-node graph as an example). We set the sampling\_ratio to 1/3, and after sorting the adjacency relationships of nodes for worker 2 and worker 3, node 4 and node 7 are selected to use predicted values for training, while other remote nodes only use historical embeddings.}
\label{fig:sampling method}
\end{figure}

As illustrated in Figure~\ref{fig:sampling method}, the full graph is partitioned across three workers. When updating Subgraph 1 (nodes {1, 2, 3}), we calculate the importance of its remote neighbors. Node 4 has the most connections to Subgraph 1, followed by node 7. With a \textbf{sampling\_ratio} of 1/3, only the top-ranked nodes (i.e., 4 and 7) are updated using the SmartNIC's predictor. The remaining remote nodes are updated using their last known historical embeddings. This selective prediction reduces the DPU's workload while focusing its computational power on the nodes most likely to affect model convergence.

\subsubsection{Communication Pipeline}\label{sec:comm:pipeline}

\begin{figure}[t]
\centering
\includegraphics[width=0.48\textwidth]{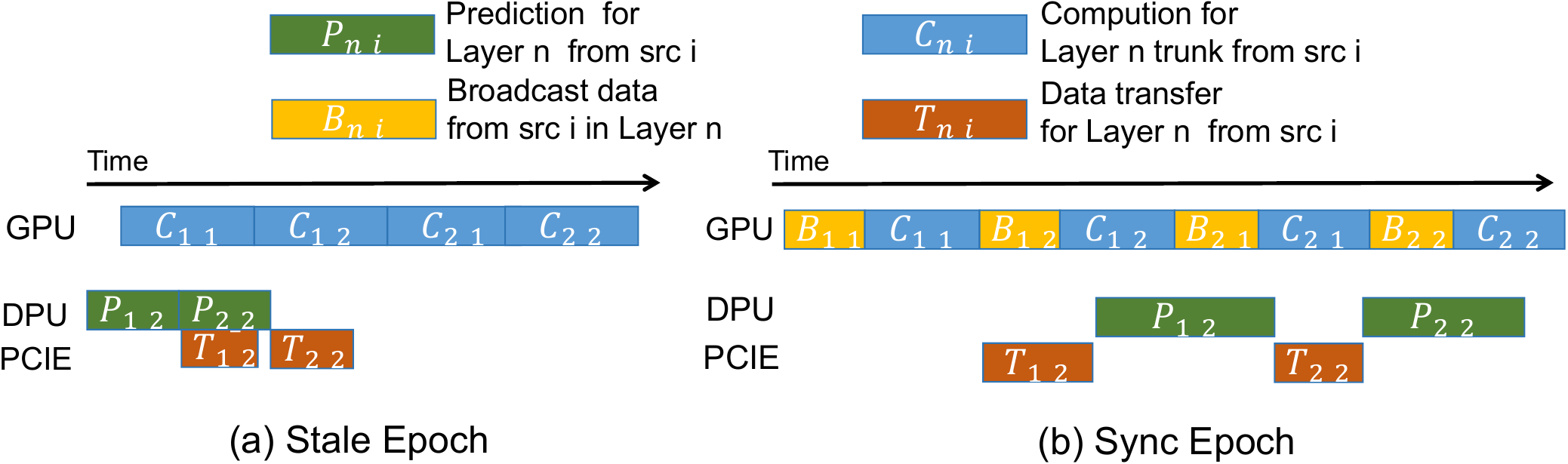}
\caption{\revision{The asynchronous communication pipeline in \sysname{} (illustrated with a 2-worker, 2-layer GNN configuration).
}}
\label{fig:data pipeline}
\end{figure}



\revision{As shown in Figure~\ref{fig:data pipeline}, we have designed specialized data slicing and asynchronous pipelines to mitigate the inherent overhead of transferring data through DPU RAM and the additional latency required for DPU-to-GPU transmission. We categorize training epochs into two distinct execution modes: \emph{Stale Epochs} and \emph{Sync Epochs}.}

\revision{In a Stale Epoch (Figure~\ref{fig:data pipeline}a), physical network communication for boundary nodes is skipped to reduce traffic volume. Instead of waiting for a broadcast, the GPU asynchronously obtains extrapolated values ($P_{n,i}$) from the DPU and merges them with the local cache. By offloading the prediction logic entirely to the SmartNIC, we ensure that the GPU's GNN training process ($C_{n,i}$) remains undisturbed, with the prediction latency hidden by the concurrent forward and backward passes.}

\revision{In a Sync Epoch (Figure~\ref{fig:data pipeline}b), the GPU performs global embedding synchronization ($B_{n,i}$) and concurrently sends the fresh embeddings of remote workers to the DPU. This allows the DPU to update the linear-trend model and generate the predicted values for the subsequent epoch. Because the DPU is positioned directly on the data path, it can provide these predictions to the GPU immediately after the broadcast is completed without introducing additional synchronization barriers. By carefully scheduling the computational workload and data transmission ($T_{n,i}$) between the DPU and GPU, our pipeline minimizes the architectural overhead introduced by utilizing the DPU as an intermediate transmission layer, ensuring that the benefits of active staleness mitigation outweigh the costs of hardware offloading.}

\subsubsection{Predictive intermediate Result Reuse}
Since the prediction process requires shifting the time window each time synchronization occurs, it necessitates repeatedly recalculating trends. Additionally, as the synchronization interval increases, intermediate historical data may be missing. If interpolation methods are used to replace the intermediate historical data, it would incur significant computational overhead. Therefore, we employ prediction records to substitute for the intermediate historical data. Moreover, each time the time window slides, some previously computed results, as well as intermediate predictions and the latest synchronization results, are overwritten. To address this algorithmic characteristic, we designed a prediction algorithm that reuses intermediate results, as shown in Algorithm~\ref{alg:self-feeding-incremental}. 

The most critical steps in the algorithm are as follows:
\textbf{i) Trend accumulators:} Trends are stored in a prefix sum-like manner. During each update, only the impact of the oldest historical record needs to be removed, and then the trend is updated using the latest embedding record.
\textbf{ii) Self-feeding:} Since the time steps of the warm-up phase and the actual training phase are different, if we want to use intermediate cached results to reduce the computational load of the predictor, historical records need to be filled. Therefore, directly using predicted values as historical records to fill the historical record table can reduce interpolation overhead.

\begin{algorithm}[t]\small
\caption{Self-Feeding Trend Update}
\label{alg:self-feeding-incremental}
\noindent\textbf{Inputs:}
Batch nodes $\mathcal{I}$; optional observations $X^{obs}\!\in\!\mathbb{R}^{|\mathcal{I}|\times d}$; mask $M^{obs}\!\in\!\{0,1\}^{|\mathcal{I}|}$; time $t$.

\noindent\textbf{State:}
Ring buffers $H\!\in\!\mathbb{R}^{N\times W\times d}$, $T\!\in\!\mathbb{Z}^{N\times W}$, write index $pos\!\in\!\{0,\dots,W-1\}^N$;\\
Last values $x^{last}\!\in\!\mathbb{R}^{N\times d}$, $t^{last}\!\in\!\mathbb{Z}^{N}$;\\
Trend accumulators $N\!\in\!\mathbb{R}^{N\times d}$, $D\!\in\!\mathbb{R}^{N}$.

\noindent\textbf{Params:}
Window $W$; decay rates $\lambda>0$, $\beta>0$; numerical $\varepsilon>0$.

\noindent\textbf{Procedure:}
\begin{enumerate}
  \item For each $i \in \mathcal{I}$:
  \begin{enumerate}
    \item $\Delta \leftarrow t - t^{last}[i]$; \; 
    \item Trend: $v \leftarrow N[i]/\max(D[i], \varepsilon)$.
    \item Predict: $x^{pred} \leftarrow x^{last}[i] + v \cdot \Delta \cdot e^{-\beta \Delta}$.
    \item Self-feed: $x^{new} \leftarrow \begin{cases}
      X^{obs}[i], & \text{if } M^{obs}[i]=1\\
      x^{pred},   & \text{otherwise}
    \end{cases}$
    \item Write ring buffer: $p \leftarrow pos[i]$; $H[i,p,:]\!\leftarrow\!x^{new}$; $T[i,p]\!\leftarrow\!t$; $pos[i]\!\leftarrow\!(p+1)\bmod W$.
    \item Decay accumulators: $\gamma \leftarrow e^{-\lambda \Delta}$; $N[i]\!\leftarrow\!\gamma N[i]$; $D[i]\!\leftarrow\!\gamma D[i]$.
    \item New slope: $s_{new} \leftarrow (x^{new} - x^{last}[i])/\max(\Delta,\varepsilon)$.
    \item Update accumulators: $N[i]\!\leftarrow\!N[i] + s_{new}$; $D[i]\!\leftarrow\!D[i] + 1$.
    \item Update last: $x^{last}[i]\!\leftarrow\!x^{new}$; $t^{last}[i]\!\leftarrow\!t$.
    \item Output $x^{pred}$ for node $i$.
  \end{enumerate}
\end{enumerate}
\end{algorithm}

\subsection{Theoretical Analysis}
\label{sec:theoretical_analysis}
We analyze the convergence of \sysname{} by bounding the prediction error and its impact on gradients. 
We assume standard bounds: $N$ nodes, dimension $d$, weight matrices bounded by $C$, and adjacency blocks bounded by $B$. The loss $\mathcal{L}$ and activations are $\rho$-smooth/Lipschitz.

\begin{theorem}[Prediction Error]
\label{thm:prediction_error}
If the second-order difference of embeddings is bounded by $\|H^{(t)}-2H^{(t-1)}+H^{(t-2)}\|_\infty\le\epsilon$, the linear-trend predictor error satisfies:
{\small \[ \|\hat H^{(t)}-H^{(t)}\|_\infty \le \tfrac{1}{2}(\Delta t)^2\epsilon =: \epsilon_H, \]}
where $\Delta t \le s$ is the staleness.
\end{theorem}
\begin{proof}[Proof Sketch]
Approximating $H^{(t)}$ via a second-order Taylor expansion around $\tau_{\text{latest}}$, the linear predictor captures the first-order trend ($\bar h \approx H^{(\tau)}$, $\Phi \approx H'^{(\tau)}$). The error is thus dominated by the second-order term $\frac{1}{2}(\Delta t)^2 H''^{(\tau)}$, bounded by $\epsilon$.
\end{proof}

Next, we bound how this error propagates to gradients. Let $p$ be the maximum communicating workers.

\begin{lemma}[Gradient Error Bound]
\label{lem:gradient_error}
The gradient estimation error is bounded by:
{\small \[ \|\nabla \hat{\mathcal{L}} - \nabla \mathcal{L}\|_\infty \le K := \rho C^2 B^2 (p + CB) \epsilon_H. \]}
\end{lemma}
\begin{proof}[Proof Sketch]
Using the triangle inequality and sub-multiplicativity on the GNN aggregation $T^{(l)} = \sum A H$, the feature error bounds are $\|\hat T - T\|_\infty \le pCB\epsilon_H$ and $\|\hat Z - Z\|_\infty \le C^2B^2\epsilon_H$. Since backpropagation distributes these errors linearly through $\rho$-Lipschitz components, the gradient difference is bounded by $K$.
\end{proof}

\begin{theorem}[Convergence]
\label{thm:convergence}
With learning rate $\eta = \min\{\frac{1}{\rho}, \frac{1}{\sqrt{N}}\}$, non-convex SGD with \sysname{} satisfies:
{\small \[ \min_{t} \mathbb{E}[\|\nabla\mathcal{L}(W_t)\|_F^2] \le \frac{2(\mathcal{L}_0 - \mathcal{L}^*) + \rho K}{\sqrt{N}} = O(1/\sqrt{N}). \]}
\end{theorem}
\begin{proof}[Proof Sketch]
This follows from standard SGD analysis with inexact gradients~\cite{peng2022sancus}. The gradient bias $K$ is constant; setting sync interval $s = O(N^{-1/4})$ ensures the error term $\epsilon_H \propto s^2$ does not dominate the $O(1/\sqrt{N})$ convergence rate.
\end{proof}

\section{Evaluation}\label{sec:eval}
This section describes our experimental setup and evaluates the system's efficiency and accuracy. We also include ablation studies to quantify the contribution of each component.
\subsection{Experiment Setup}
	\textbf{Testbed.} Experiments were performed on a cluster of eight GPU nodes. Each node is equipped with two Intel Gold CPUs (v5 series, 16 cores, 2.3GHz), 192 GB of DRAM, one NVIDIA V100 GPU, and one NVIDIA BlueField-3 DPU (SmartNIC) running CentOS 7. The network provides 10 Gbps bandwidth. We use CUDA 12.2, DOCA 3.0, OpenMPI 4.0.5, PyTorch v2.5.1, and cuDNN 9.1.0 in our software stack. 

\begin{table}[t] 
\small 
\centering
\caption{Graph Datasets}
\resizebox{0.5\textwidth}{!}{	
\begin{tabular}{ccccc} 
\toprule 
Name & Nodes & Edges & Features & Classes \\
\midrule  
 Reddit & 232,965 & 11,606,919 & 602 & 41 \\
 IGB-small & 1,000,000 & 12,070,502 & 128 & 19\\
 ogbn-products & 2,449,029 & 61,859,140 & 100 & 47 \\
 IGB-medium & 10,000,000 & 120,077, 694 & 128 & 19\\
\bottomrule 
\end{tabular}
}
\setlength{\abovecaptionskip}{5pt} 
\label{tab:Graph Datasets} 
\end{table}

\textbf{Datasets.} Table~\ref{tab:Graph Datasets} lists the datasets used in our experiment, including Reddit~\cite{hamilton2017inductive}, IGB-small~\cite{Khatua2023IGBAT}, ogbn-products~\cite{2020arXiv200500687H} and IGB-medium~\cite{Khatua2023IGBAT}. These are widely adopted datasets in GNN training research and are presented in ascending order of scale.

\revision{
We evaluate \sysname{} across several GNN architectures with varying depths and computational complexities: 1) We use 3-layer GCN~\cite{kipf2016semi} as our primary benchmarks for standard full-graph training comparison. To assess the system's robustness against staleness accumulation in deeper networks, we extend our evaluation to include 4-layer and 6-layer GCNs. 2) We implement a light variant of the Graph Attention Network (GAT)~\cite{velivckovic2017graph} and GraphSAGE~\cite{hamilton2017inductive}. Unlike the static weighting in GCN, GAT involves dynamic, learnable attention coefficients, which test the linear-trend predictor's ability to handle more complex embedding evolution patterns.
For all models, we set the hidden dimension size to 16 and use the Adam optimizer. We follow the standard training-validation-test split provided by the respective dataset creators.}

\textbf{Comparisons.} To demonstrate the generality and effectiveness of \sysname{}, we integrate our in-network prediction mechanism into two state-of-the-art full-graph GNN training systems with distinct communication-reduction strategies: 
SANCUS~\cite{peng2022sancus} and NeutronTP~\cite{NeutronTP}.
SANCUS leverages historical embeddings to reduce communication frequency but lacks explicit staleness compensation, making it susceptible to accuracy degradation from embedding drift. NeutronTP employs tensor parallelism and computation/communication overlap for performance scaling but maintains synchronous execution without stale embeddings. These complementary baselines allow us to evaluate \sysname{}'s prediction benefits across different architectural paradigms. Additionally, we compare against BNS-GCN~\cite{Wan2022BNSGCNEF}, a sampling-based approach that reduces communication through random boundary node sampling, providing perspective against non-full-graph methods.

\textbf{Configurations.} 
Our evaluation employs a structured experimental design to comprehensively assess \sysname{} across multiple dimensions. 

First, we conduct an \textbf{overall performance comparison} using three datasets (Reddit, IGB-small, and ogbn-products) that are directly compatible with all baseline methods. The IGB-medium dataset required graph coarsening for SANCUS to avoid Out-of-Memory errors, and due to framework incompatibilities with BNS-GCN's sampling-based approach, we excluded it from this specific multi-system comparison to ensure a fair, non-coarsened evaluation (Section~\ref{sec:overall:results}). 

For \textbf{scalability analysis}, we focus exclusively on full-graph systems (SANCUS and NeutronTP) across all four datasets. We scale the system from 4 to 16 GPUs to observe how the SNI-GNN pipeline maintains throughput as inter-node communication pressure increases (Section~\ref{sec:scalability}).

\revision{For \textbf{sensitivity and 
robustness studies}, we perform detailed parameter studies to isolate the impact of synchronization intervals and sampling rates on both model accuracy and system efficiency. To demonstrate architectural generalizability, we extend our evaluation beyond GCN to include GraphSAGE and GAT variants, including deeper architectures (4--6 layers) to test robustness against staleness accumulation (Section~\ref{sec:sensitivity}).}

All experiments are evaluated across four key metrics: total training time, communication reduction effectiveness, prediction accuracy, and training convergence behavior, providing a comprehensive assessment of \sysname{}'s performance characteristics. In Section~\ref{sec:accuracy}, we explicitly study the accuracy impact of staleness techniques.

\subsection{Overall Comparison Results}\label{sec:overall:results}
\begin{figure}[t]
    \centering
    \includegraphics[width=0.5\textwidth]{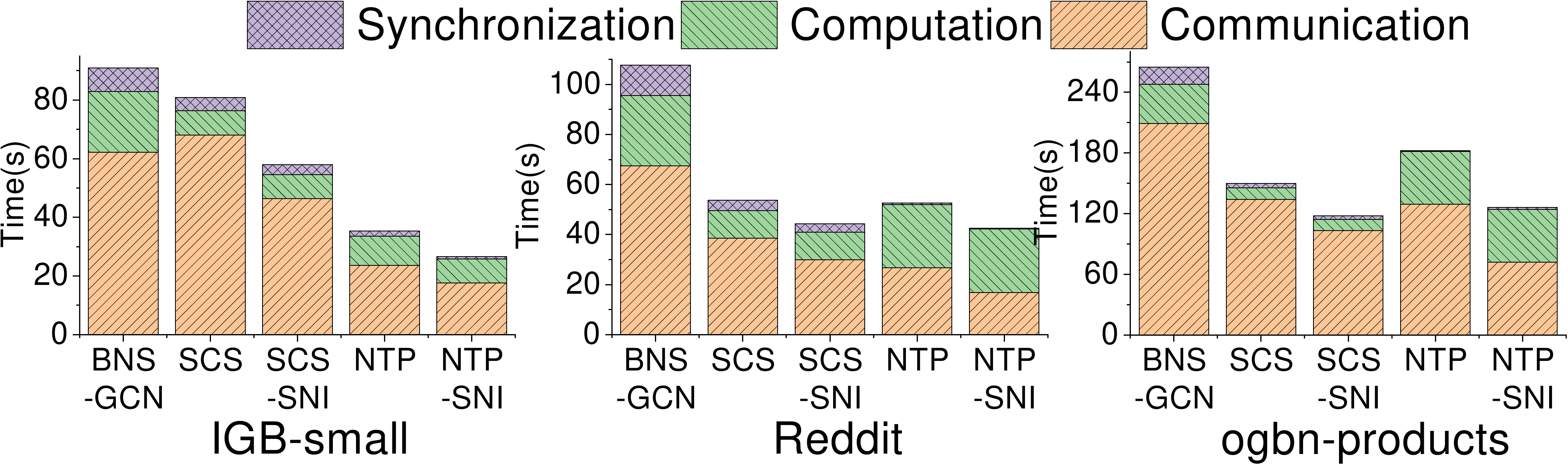}
    \caption{Runtime breakdown (synchronization, computation, communication) using 8 nodes for 3-layer GCN (hidden=16) on Reddit/IGB-small/ogbn-products datasets. Methods: BNS-GCN, SCS (SANCUS), SCS-SNI (SANCUS+SNI), NTP (NeutronTP), NTP-SNI (NeutronTP+SNI).}
    \label{fig:overall}
\end{figure}

In the overall comparison, all methods were trained for 300 epochs using a 3-layer GCN~\cite{kipf2016semi} with a 20-epoch warm-up period, consistent with established practices for staleness-tolerant training. Figure~\ref{fig:overall} presents results where \sysname{} maintains accuracy loss within 0.01 across all datasets, using sampling rates of 0.3, 0.1, and 0.05 for Reddit, IGB-small, and ogbn-products, respectively. At these conservative sampling rates, the DPU-based predictor operates transparently alongside GPU training without introducing observable overhead.

Our analysis reveals that SANCUS experiences significant accuracy degradation when communication intervals exceed three epochs, whereas \sysname{} enables safe extension to five intervals while preserving model quality. This translates to communication reductions of 28.97\%, 31.86\%, and 23.2\% across the three datasets, respectively. \sysname{} also substantially benefits NeutronTP, achieving communication reductions of 37.45\%, 25.42\%, and 44.3\% despite its architectural constraints on stale execution.

The end-to-end performance gains are equally compelling: SNI-enhanced SANCUS achieves speedups of {{2.34$\times$}, {1.56$\times$}, and {1.73$\times$} over BNS-GCN, while SNI-enhanced NeutronTP delivers even greater improvements of {2.51$\times$}, {3.55$\times$}, and {1.60$\times$} over the sampling-based baseline.} These results demonstrate \sysname{}'s ability to enhance diverse training frameworks while maintaining strict accuracy guarantees.

\begin{figure}[t]
    \centering
    \includegraphics[width=0.48\textwidth]{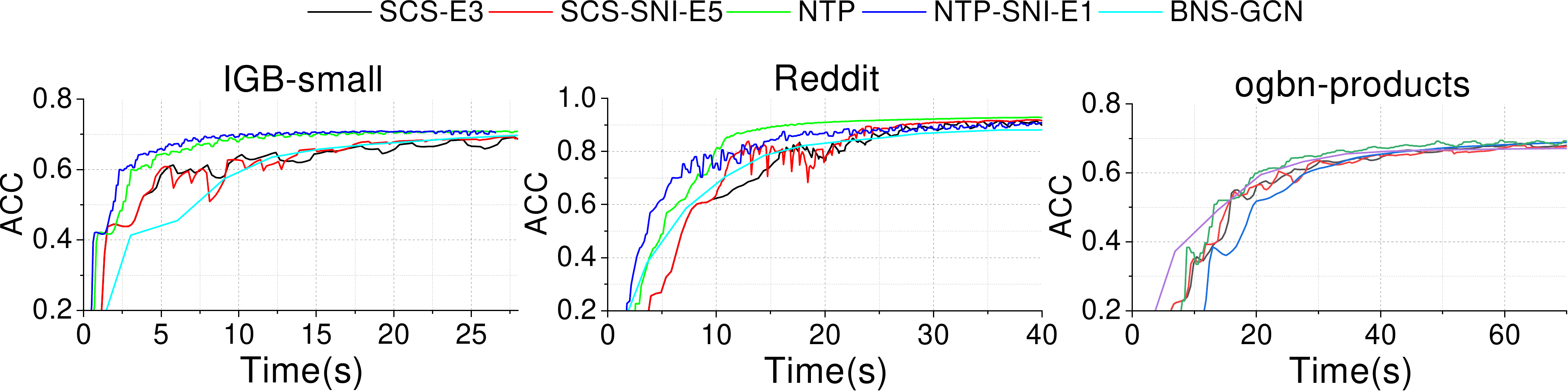}
    \caption{The accuracy results using 8 nodes for 3-layer GCN (hidden=16) on Reddit/ogbn-products/IGB-small datasets. The prediction sampling rates for the three datasets are 0.3, 0.1 and 0.05. E1/3/5 represents performing synchronous communication once every 1/3/5 iterations during training.}
    \label{fig:overall accuracy}
\end{figure}

Figure~\ref{fig:overall accuracy} demonstrates \sysname{}'s impact on training convergence and final accuracy across all experimental configurations. While the initial warm-up period produces overlapping accuracy curves across methods, \sysname{}'s advantages become pronounced in the critical later training stages. SANCUS with SNI prediction exhibits significantly improved stability and accelerated convergence, maintaining smoother optimization trajectories compared to its baseline counterpart. NeutronTP similarly benefits from SNI integration, achieving faster convergence with only minimal accuracy reduction ($\leq$ 0.01) on the Reddit dataset. Importantly, SNI-enhanced training consistently reaches target accuracy levels more rapidly than standard approaches, demonstrating that the prediction mechanism effectively preserves model quality while accelerating the training process through reduced synchronization overhead.


\begin{figure*}[t]
    \centering
    \includegraphics[width=\textwidth]{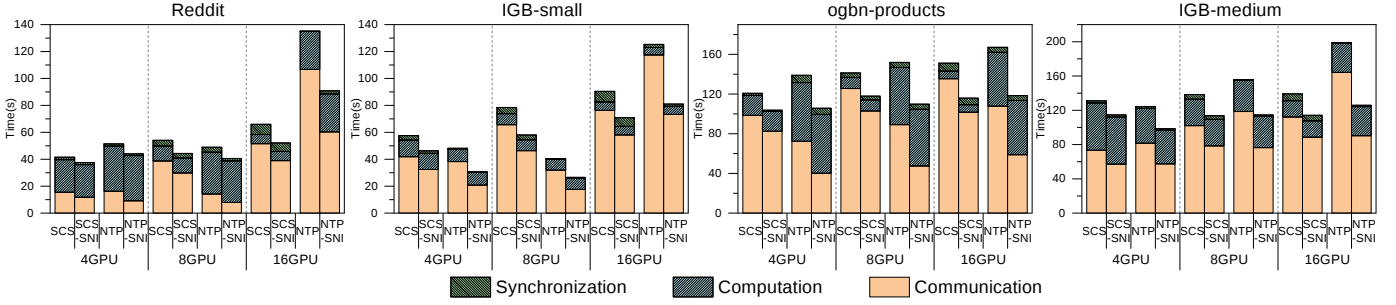}
    \caption{Breakdown of time overhead (synchronization, computation, communication) for SCS, SCS-SNI, NTP, and NTP-SNI Under 4/8/16 GPUs across four datasets (Reddit, IGB-small, ogbn-products, IGB-medium).}
    \label{fig:scalability_study}
\end{figure*}

\subsection{Scalability Study}
\label{sec:scalability}

\begin{figure}[t]
    \centering
    \includegraphics[width=0.48\textwidth]{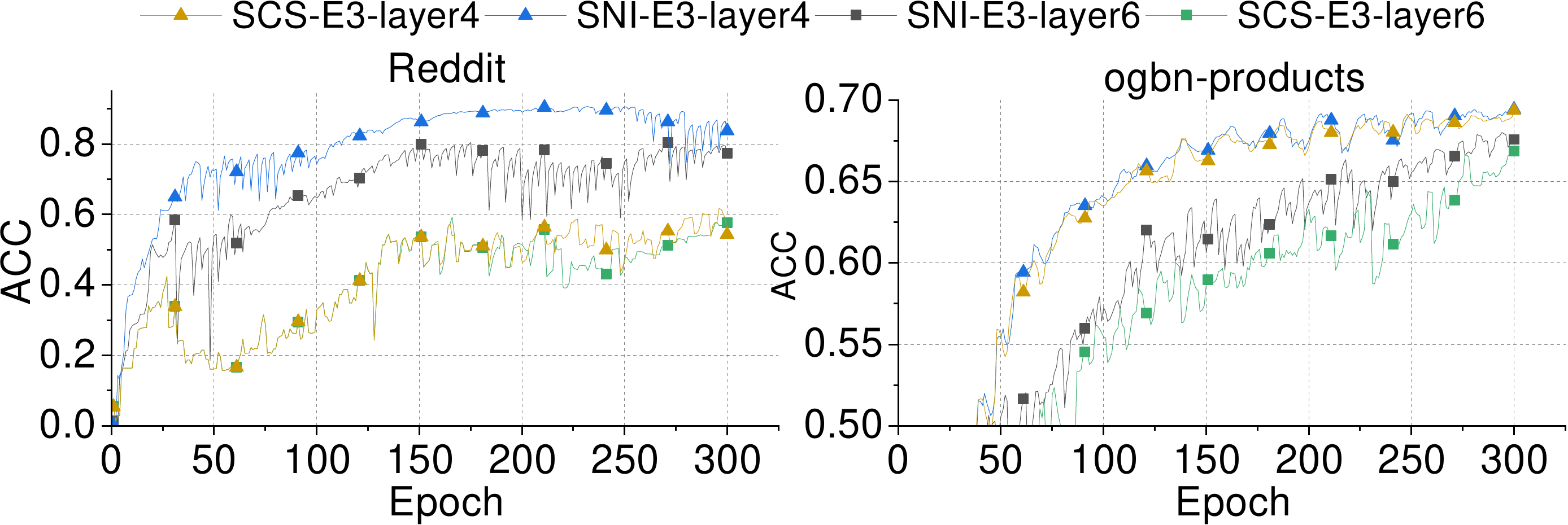}
    \caption{\revision{Deeper-model study (4/6 layers): convergence under multi-node training.}}
    \label{fig:deep_model_scalability}
\end{figure}

We evaluate \sysname{}'s scaling behavior by training a 3-layer GCN across cluster sizes from $4$ to $16$ GPUs, with detailed time breakdowns shown in Figure~\ref{fig:scalability_study}. We decompose the time into three parts, including communication, computation and synchronization.
Our analysis reveals three critical patterns that emerge as system scale increases.

\textbf{(1) Communication Dominance at Scale}: Across all datasets, communication becomes the primary bottleneck at $16$ GPUs, consuming $81\%$ of total time for IGB-medium and $65\%$ for ogbn-products. This escalating overhead underscores the fundamental importance of communication optimization in distributed GNN training.

\textbf{(2) Consistent Communication Reduction}: 
\sysname{} delivers substantial communication savings across all scales and datasets. On Reddit, SCS-SNI reduces communication time by $24\%$ at $16$ GPUs, yielding a {$1.21\times$ end-to-end speedup}. The benefits are even more pronounced on ogbn-products, where NTP-SNI halves communication costs ($107.9\text{s} \rightarrow 58.9\text{s}$) and achieves {a $1.29\times$ overall speedup}. Notably, even on IGB-small where absolute gains appear modest, SCS-SNI maintains $24\%$ communication reduction, showing consistent effectiveness.

\textbf{(3) Architectural Generality}: \sysname{} complements both major parallelization paradigms. NTP-SNI achieves maximum gains on ogbn-products (tensor-parallel optimized), while SCS-SNI excels on IGB-medium (partition-parallel optimized), proving the framework's adaptability to diverse system architectures. The boundary-importance sampling remains effective regardless of graph scale, with larger datasets benefiting more significantly due to their inherently higher communication overhead.

\revision{We further evaluate \sysname{} under \textbf{deeper GNN} backbones (4-layer and 6-layer) to test robustness against stronger staleness amplification in multi-hop message passing.
As shown in Figure~\ref{fig:deep_model_scalability}, 
while SANCUS exhibits significant accuracy fluctuations and a noticeable decline in convergence stability as depth increases to 6 layers, \sysname{} maintains a much smoother and higher optimization trajectory. 
This robustness is achieved because our in-network predictor actively compensates for the staleness of boundary embeddings on the DPU before they are utilized in the GPU's aggregation kernels. By providing a fresher approximation of remote embeddings at each step, SNI-GNN effectively truncates the error propagation chain. 
}

The scalability results demonstrate that \sysname{}'s prediction mechanism effectively addresses the fundamental scaling challenge in distributed GNN training: communication costs grow disproportionately with system size, and our in-network approach provides a scalable solution that maintains efficiency across diverse configurations.

\subsection{Accuracy and Convergence Analysis}\label{sec:accuracy}
Since the staleness of embeddings can greatly affect the convergence of model training and the final accuracy, we conducted a detailed study on the model accuracy, sampling rate, and communication interval. Since we used the simplest GCN model architecture and did not add randomness such as dropout, our model training accuracy did not fluctuate. Therefore, we can confirm that the improvement in accuracy comes from improved prediction.

\subsubsection{Impact of Prediction Sampling Rate}
We evaluate the sensitivity of our system to the Sampling Rate (SR) using SANCUS with a communication interval of 5 as a baseline.
We vary the SR from 0.01, 0.05, 0.1 to 0.3 to study its impact on \sysname{}. As shown in Table~\ref{tab:samp_acc_comparison}, peak accuracy is positively correlated with the sampling rate across all datasets. For example, on ogbn-products, increasing the SR from 0 to 0.3 yields a 1.00\% accuracy gain.

\begin{table}[t]\footnotesize
    \centering
    \caption{Accuracy results across datasets and Sampling Rates (SR).}
    \setlength{\tabcolsep}{3pt}
    \begin{tabular}{c c c c c}
        \toprule
        SR & Reddit & IGB-small & ogbn-products & IGB-medium \\
        \midrule
        0    & 91.19\% & 69.84\% & 69.54\% & 60.81\% \\
        0.01 & 91.35\% (\textbf{+0.16}) & 69.88\% (\textbf{+0.04}) & 69.91\% (\textbf{+0.37}) & 60.76\% (-0.05) \\
        0.05 & 91.77\% (\textbf{+0.58}) & 69.94\% (\textbf{+0.10}) & 70.19\% (\textbf{+0.65}) & 61.45\% (\textbf{+0.64}) \\
        0.1  & 91.88\% (\textbf{+0.69}) & 70.12\% (\textbf{+0.28}) & 70.23\% (\textbf{+0.69}) & 61.65\% (\textbf{+0.84}) \\
        0.3  & 92.06\% (\textbf{+0.87}) & 70.36\% (\textbf{+0.52}) & 70.54\% (\textbf{+1.00}) & 62.21\% (\textbf{+1.40}) \\
        \bottomrule
    \end{tabular}
    \label{tab:samp_acc_comparison}
\end{table}

While higher sampling rates improve accuracy by providing more comprehensive neighborhood updates, they also increase the DPU's memory footprint for historical records and computational load. In practice, the optimal SR is determined by the system's "idle window". That is, the sampling rate should be maximized as long as the DPU's prediction latency remains hidden behind the GPU's computation, as illustrated in our pipeline design (Figure~\ref{fig:data pipeline}).

\begin{figure}[t]
    \centering
    \includegraphics[width=0.45\textwidth]{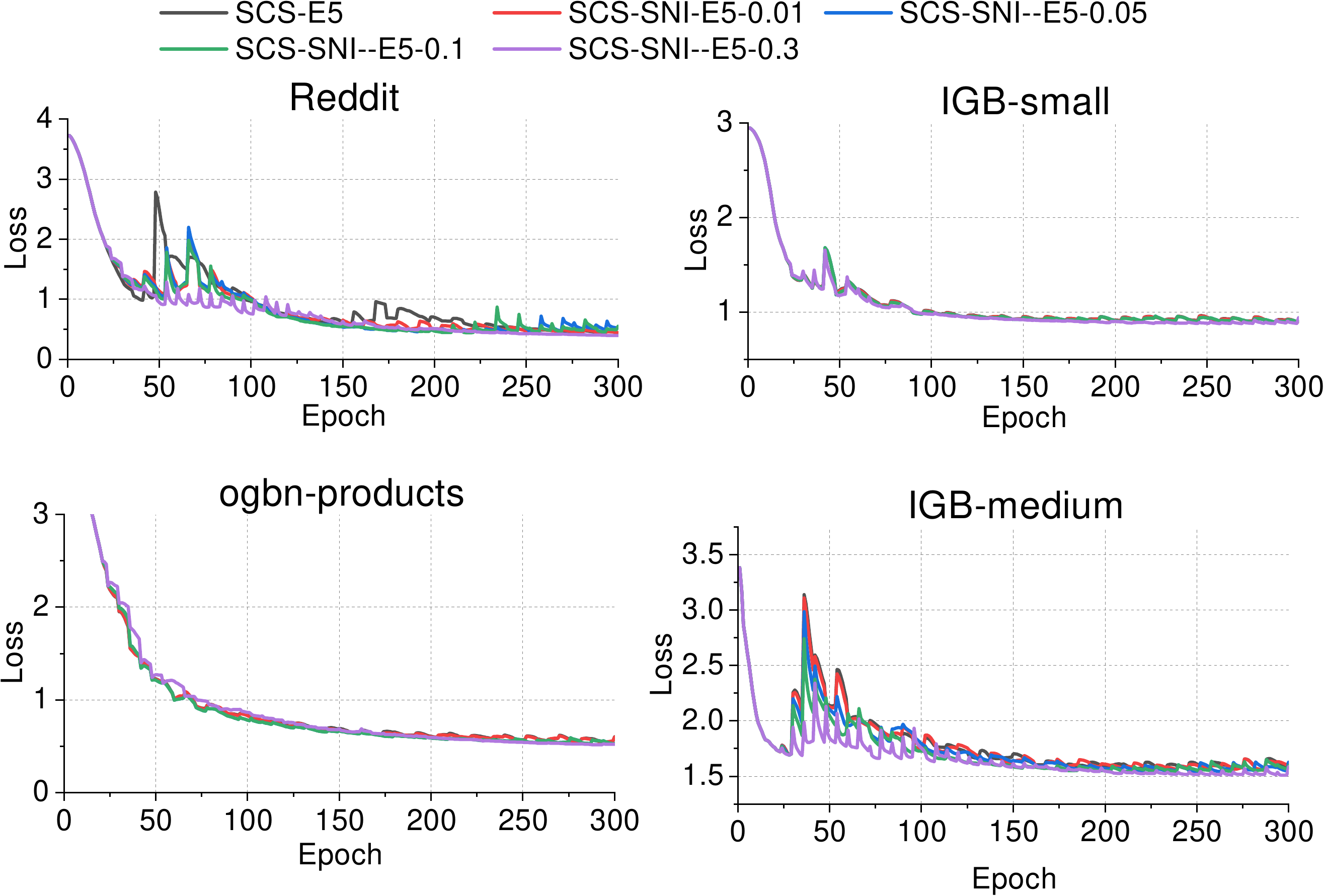}
    \caption{Training loss curves of SANCUS with SNI prediction and its variants with different sampling rates across datasets. }
    \label{fig:loss_curve}
\end{figure}

\subsubsection{Convergence Stability}
Beyond peak accuracy, we examine training convergence via the loss curves in Figure~\ref{fig:loss_curve}. 
First, all variants with the sampling predictor, namely SCS-E5-0.01, SCS-E5-0.05, SCS-E5-0.1, and SCS-E5-0.3, exhibit loss curves that either match or outperform the baseline SCS-E5 in terms of convergence speed and stability. For instance, on the IGB-small dataset, the variants achieve a stable low loss earlier than the baseline. On ogbn-products, they show a smoother and more consistent decline in loss. Even on the Reddit and IGB-medium datasets, where the baseline already has a decent convergence trend, the variants maintain comparable performance while offering more flexibility in parameter adjustment.

\subsubsection{Communication Interval Study}
To verify the communication-avoiding capability of \sysname{}, we compare it against various synchronization intervals ($E \in \{1, 3, 5\}$). We analyzed two key metrics, including total training time and final accuracy. The results are shown in Table~\ref{tab:interval_study}. Note that for each dataset, we configured \sysname{} with the maximum SR that fits within the system's idle window (0.3 for Reddit, 0.1 for IGB-Small, and 0.05 for ogbn-products), ensuring that the in-network prediction does not introduce additional latency to the critical path.

Across all datasets, increasing the interval from 1 to 5 reduces training time by approximately 43--51\% due to lower synchronization frequency. When operating at $E=5$, SNI-GNN provides a superior Pareto frontier: it retains the 40\%+ time savings of a large interval while recovering significant accuracy loss. On Reddit, SNI-GNN recovers nearly 0.9\% accuracy compared to the stale baseline, keeping the final accuracy loss within 1 percentage point of the ideal synchronous training ($E=1$).
%

\begin{table}[t]
  \centering
  \caption{{Impact of Communication Interval on Training Time and Accuracy. SR denotes the prediction Sampling Rate.}}
  \small 
  \renewcommand{\arraystretch}{0.95} 
  \setlength{\tabcolsep}{3pt}
  \resizebox{\columnwidth}{!}{
  \begin{tabular}{lcccccc}
    \toprule
    \multirow{2}{*}{\textbf{Method}} & \multicolumn{2}{c}{\textbf{Interval=1}} & \multicolumn{2}{c}{\textbf{Interval=3}} & \multicolumn{2}{c}{\textbf{Interval=5}} \\
    \cmidrule(lr){2-3} \cmidrule(lr){4-5} \cmidrule(lr){6-7}
     & Time (s) & Acc (\%) & Time (s) & Acc (\%) & Time (s) & Acc (\%) \\
    \midrule
    \multicolumn{7}{l}{\textit{\textbf{Reddit} (SR=0.3)}} \\
    SANCUS          & 77.61 & 93.01 & 53.81 & 92.28 & 44.31 & 91.19 \\
    \textbf{\sysname{}} & 79.36 &93.36 & 57.23 & 92.74 & 46.59 & 92.06 \\
    \midrule
    \multicolumn{7}{l}{\textit{\textbf{IGB-Small} (SR=0.1)}} \\
    SANCUS          & 119.12 & 71.13 & 80.87 & 70.33 & 58.02 & 69.84 \\
    \textbf{\sysname{}} & 123.59 & 71.21 & 81.64 & 70.47 & 60.41 & 70.12 \\
    \midrule
    \multicolumn{7}{l}{\textit{\textbf{ogbn-products} (SR=0.05)}} \\
    SANCUS          & 231.39 & 71.03 & 149.75 & 69.65 & 117.89 & 69.47 \\
    \textbf{\sysname{}} & 234.58 & 70.88 & 154.31 & 70.08 & 123.64 & 70.19 \\
    \bottomrule
  \end{tabular}
  }
  \label{tab:interval_study}
\end{table}

\subsubsection{Robustness under Non-smooth Dynamics}
\revision{Our predictor's theoretical effectiveness rests on the assumption of bounded second-order embedding dynamics. To evaluate the system’s resilience when this assumption is challenged, we conducted a stress test by intentionally introducing non-smooth embedding evolution through two scenarios: (i) early-epoch regimes prior to the completion of the 20-epoch warm-up, and (ii) unstable optimization settings using significantly larger learning rates.}

\revision{As shown in Figure~\ref{fig:smoothness_break}, although these volatile conditions naturally lead to higher initial prediction errors compared to our default settings, SNI-GNN maintains convergence without catastrophic accuracy loss. The combination of active DPU prediction with periodic global synchronization acts as a fail-safe that prevents unbounded error growth even during temporary violations of embedding smoothness. These results confirm that SNI-GNN is robust enough to handle the non-linear dynamics common in the early stages of GNN training.}

\begin{figure}[t]
    \centering
    \includegraphics[width=0.48\textwidth]{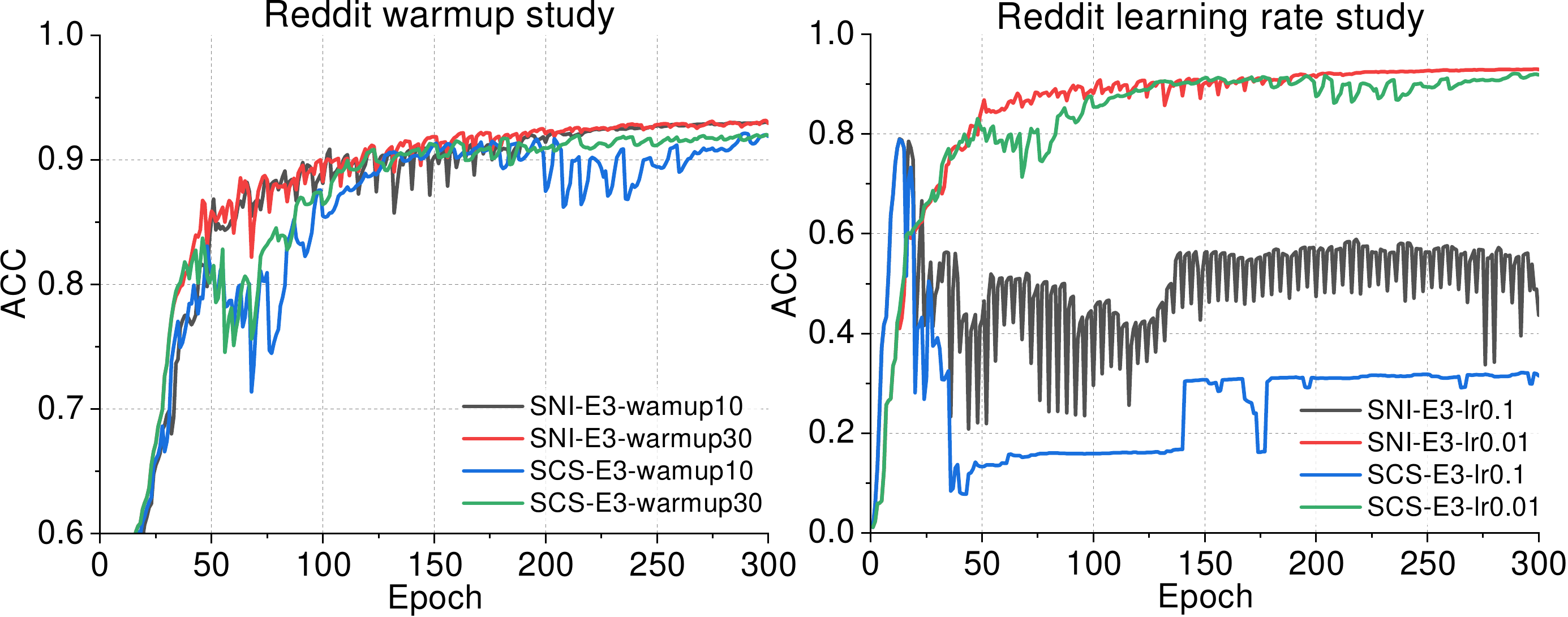}
    \caption{\revision{Robustness under non-smooth dynamics.}}
    \label{fig:smoothness_break}
\end{figure}

\subsection{Sensitivity and Robustness Study}
\label{sec:sensitivity}
For SNI prediction, any message-passing-based model can theoretically be used. As long as we can identify fixed message communications, we can replace the communication content with a historical cache and use a predictor to make short-term predictions on the cache, thereby reducing the communication frequency during model training. To verify the generalizability of our method, we also conducted a simple test on the \revision{GAT~\cite{velivckovic2017graph}, and GraphSAGE~\cite{hamilton2017inductive} models}, as shown in Figure~\ref{fig:model_loss_curve}.

\begin{figure}[t]
    \centering
    \begin{subfigure}[t]{0.49\textwidth}
        \centering
        \includegraphics[width=\textwidth]{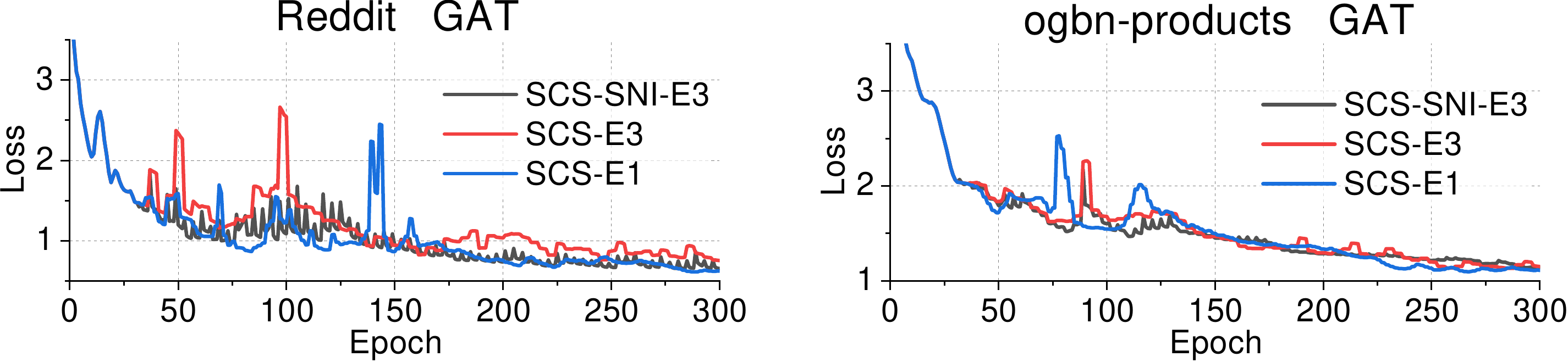}
        \label{subfig:gat-loss}
    \end{subfigure}
    \hfill
    \begin{subfigure}[t]{0.49\textwidth}
        \centering
        \includegraphics[width=\textwidth]{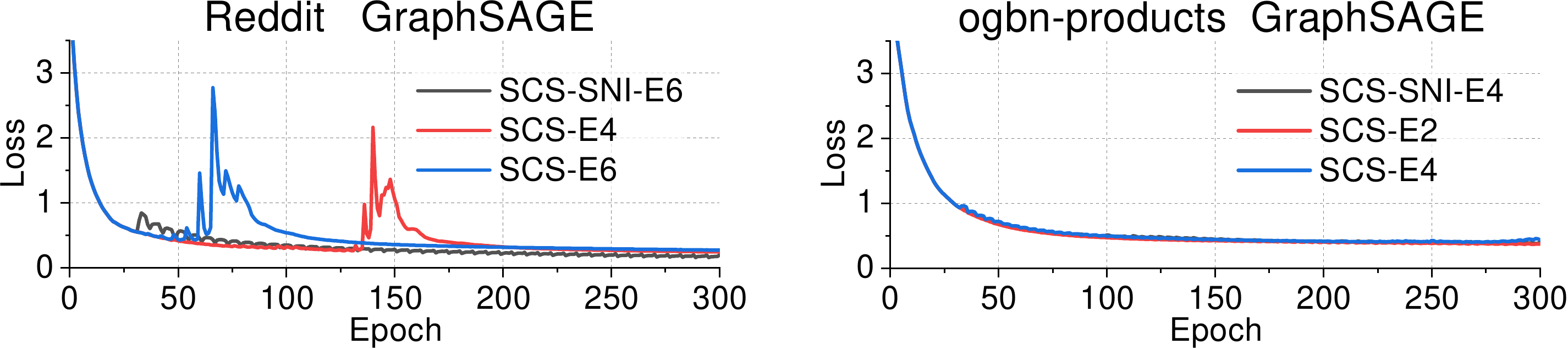}
        \label{subfig:graphsage-loss}
    \end{subfigure}
    \vspace{-2ex}
    \caption{\revision{Training loss curves of SANCUS with SNI prediction on GAT/GraphSAGE models under 8 GPUs. E1/2/3/4/6 denotes the communication intervals of 1/2/3/4/6.}}
    \label{fig:model_loss_curve}
\end{figure}

\revision{The figure shows that SNI-GNN performs robustly across both architectures. For the GAT model, which introduces dynamic, learnable attention coefficients, our predictor achieves nearly identical convergence speed and final loss at a communication interval of 3 ($E=3$) compared to the ideal synchronous baseline ($E=1$). Similarly, for GraphSAGE, the SNI-enabled training closely tracks the training dynamics of the non-stale baseline even as effective communication intervals are increased.}


\subsection{Ablation Study}
\label{sec:ablation}

\revision{To quantify the isolated impact of each optimization within \sysname{}, we conduct an incremental ablation study focusing on three core components: (i) importance-based boundary-node sampling, (ii) the self-feeding mechanism (intermediate-result reuse), and (iii) the asynchronous DPU-GPU pipeline.
As summarized in Table~\ref{tab:ablation_optimizations_transposed}, our results demonstrate how these optimizations cumulatively overcome DPU resource constraints to improve end-to-end training efficiency.}

\subsubsection{Impact of Importance Sampling}
\revision{Importance sampling serves as the fundamental enabler for DPU-offloading. Without sampling, the DPU must predict embeddings for the entire boundary set, which is computationally prohibitive for the SmartNIC's onboard ARM cores. On datasets like \textit{ogbn-products}, full-graph prediction results in execution times 2--10$\times$ longer than the sampled baseline. Consequently, we fix the sampling rate for each dataset (marked in Table~\ref{tab:ablation_optimizations_transposed}) and evaluate the remaining optimizations relative to this feasible baseline.}

\subsubsection{Self-Feeding and Asynchronous Pipeline}
\revision{The self-feeding mechanism (intermediate-result reuse) provides a consistent performance boost of 10--20\% across all datasets. By leveraging previously computed trends, it minimizes redundant arithmetic operations on the DPU. The most significant speedup, however, is realized by the asynchronous pipeline. By overlapping the DPU's prediction and PCIe transfer tasks ($P_{n,i}$ and $T_{n,i}$ in Figure~\ref{fig:data pipeline}) with the GPU's computation ($C_{n,i}$), we successfully hide the prediction latency. For instance, on the \textit{Reddit} dataset, the asynchronous pipeline reduces the epoch time from 83.17s to 44.31s, nearly doubling the training throughput. This empirical evidence validates that the synergy between algorithm simplification and pipeline overlapping is essential for effective SmartNIC-assisted GNN training.}



\begin{table}[t]\small
    \centering
    \caption{\revision{Cumulative impact of our three primary optimizations on end-to-end time (s).}}
    \setlength{\tabcolsep}{5pt}
    \begin{tabular}{lcccc}
        \toprule
        Method & IGB-small & Reddit  & ogbn-prod. & IGB-med. \\
        \midrule
        + sampling              & 117.58 & 106.52 & 275.20 & 180.21 \\
        + self-feeding           &  86.43 &  83.17 & 210.30 & 161.12 \\
        + async pipeline           &  58.02 &  44.31 & 117.89 & 107.72 \\
        \bottomrule
    \end{tabular}
    \label{tab:ablation_optimizations_transposed}
\end{table}


\begin{figure}[t]
    \centering
    \includegraphics[width=0.48\textwidth]{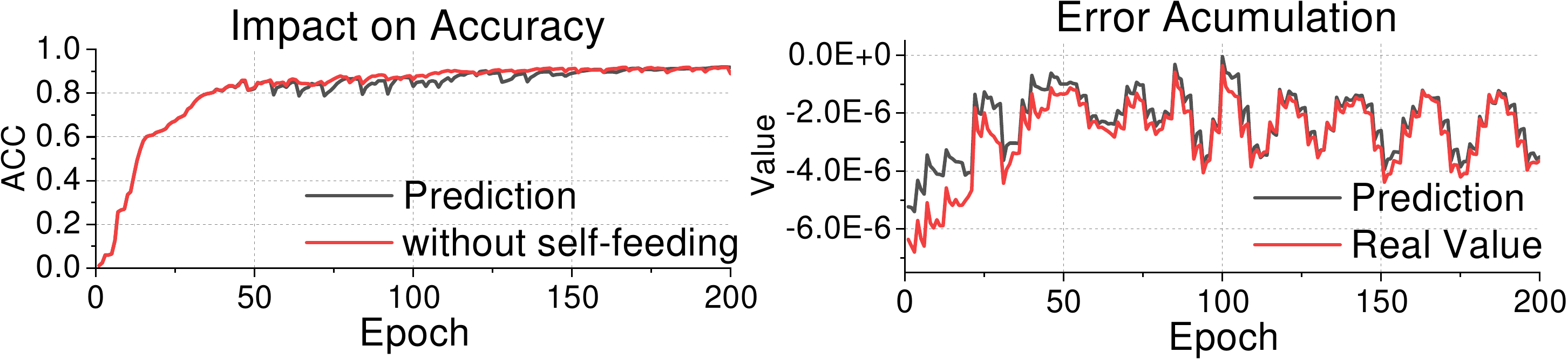}
    \caption{\revision{Prediction error over epochs with/without self-feeding on Reddit Dataset.}}
    \label{fig:prediction_error_epochs}
\end{figure}

\subsection{Microbenchmark: CPU vs. DPU Offloading}
\label{sec:cpu_vs_dpu}


To justify the architectural complexity of SmartNIC offloading, we compare predictor execution on the host CPU and the DPU from both topology and latency perspectives. Although the CPU provides stronger raw compute capability and more flexible programmability, CPU-based offloading is more vulnerable to host-side PCIe contention in modern multi-GPU servers. In contrast, DPU offloading places the predictor directly on the communication path and avoids routing intermediate embeddings through the host, making it a more natural fit for distributed full-graph GNN training.

We first analyze the communication topology using a Dell PowerEdge XE9680-style PCIe Gen5 hierarchy with 2:1 switch-uplink oversubscription, where each switch serves two GPUs and two BlueField-3 DPUs through a shared $\times$16 uplink to the CPU root complex. Under this topology, CPU-based offloading forces GPU-bound traffic to compete on shared host uplinks, whereas DPU-based offloading keeps communication largely pair-local between each GPU and its attached SmartNIC. As summarized in Table~\ref{tab:pcie_contention}, when 8 GPUs communicate concurrently, the effective per-GPU bandwidth of the CPU path drops to 32.0 GB/s, while the DPU path remains near line rate at 63.0 GB/s. This gap suggests CPU offloading's growing disadvantage with communication concurrency.

\begin{table}[t]
\centering
\caption{\revision{PCIe bandwidth contention: CPU vs. DPU communication patterns (theoretical, Dell XE9680 PCIe Gen5 with 2:1 switch-uplink oversubscription).}}
\label{tab:pcie_contention}
\footnotesize  
\setlength{\tabcolsep}{3pt}  
\renewcommand{\arraystretch}{0.95}
\begin{tabular}{@{}lcc@{}}  
\toprule
\textbf{Pattern} & \textbf{\#GPU} & \textbf{BW/GPU (GB/s)} \\
\midrule
CPU $\rightarrow$ GPU (1:1) & 4  & 63.0 \\
DPU $\rightarrow$ GPU (1:1) & 4  & 63.0 \\
\midrule
CPU $\rightarrow$ GPU (1:2) & 8  & 32.0 \\
DPU $\rightarrow$ GPU (1:1) & 8  & 63.0 \\
\bottomrule
\end{tabular}
\end{table}

\revision{We further dissect the performance of offloading the prediction algorithm to the CPU versus the DPU. Table~\ref{tab:offload_breakdown} breaks down the processing time into four components: calculation during communication rounds and calculation during stale (prediction) rounds, for both CPU and DPU targets.}

\revision{The breakdown includes: (1) \textbf{CPU-Sync}: Computation time during communication rounds when offloaded to CPU; (2) \textbf{CPU-Stale}: Computation time during stale rounds on CPU; (3) \textbf{DPU-Sync}: Computation time during communication rounds on DPU; (4) \textbf{DPU-Stale}: Computation time during stale rounds on DPU. While the CPU offers stronger raw computing power, DPU offloading eliminates significant data movement overhead and host interference, providing a more predictable and scalable execution environment for distributed training.}

\begin{table}[t]
    \centering
    \caption{\revision{Offloading overhead breakdown across datasets (per-offload latency, $\mu$s). CPU/DPU-Sync denotes predictor-side computation during synchronization (communication) rounds; CPU/DPU-Stale denotes predictor-side computation during stale (prediction) rounds.}}
    \small
    \setlength{\tabcolsep}{4pt}
    \renewcommand{\arraystretch}{0.95}
    \resizebox{\columnwidth}{!}{
    \begin{tabular}{lcccc}
        \toprule
        \textbf{Dataset} & \textbf{CPU-Sync}  & \textbf{DPU-Sync} & \textbf{CPU-Stale} & \textbf{DPU-Stale} \\
        \midrule
        Reddit        & 13,692.9 & 15,961.9 & 1,113.1 & 680.6 \\
        IGB-small     & 19,940.5 & 23,126.7& 2,365.6 & 1281.3 \\
        ogbn-products & 24,760.1 & 28,568.0 & 2,813.2 & 1683.4 \\
        \bottomrule
    \end{tabular}
    }
    \label{tab:offload_breakdown}
\end{table}


\section{Related Work}

\textbf{Full Graph Training.} Full-graph training avoids sampling bias but suffers from massive cross-partition communication. Existing countermeasures include overlapping computation and communication (\textit{e.g.}, FlexGraph~\cite{wang2021flexgraph}), minimizing volume via tensor parallelism (NeutronTP~\cite{NeutronTP}), and adopting adaptive quantization (SYLVIE~\cite{zhang2024sylvie}, AdaQP~\cite{wan2023adaptive}). Other approaches reduce overhead by sampling boundary neighbors (BNS-GCN~\cite{Wan2022BNSGCNEF}) or lowering frequency by caching historical embeddings and skipping broadcasts (SANCUS~\cite{peng2022sancus}).

\textbf{Historical Embeddings and Staleness.} Leveraging historical embeddings is a prevalent strategy to alleviate communication bottlenecks (PipeGCN~\cite{wan2022pipegcn}, GNNAutoScale~\cite{fey2021gnnautoscale}). To balance staleness and accuracy, recent works use feature momentum (GraphFM~\cite{yu2022graphfm}), selective caching via gradient/staleness criteria (FreshGNN~\cite{huang2023freshgnn}), or temporal dynamic modeling (SAT~\cite{bai2025staleness}). Unlike these software-only approaches that utilize host resources, \sysname{} introduces active in-network mitigation to resolve staleness without host contention.

\textbf{SmartNIC/DPU Offloading.} DPUs provide programmable in-network compute to relieve host contention on critical data paths. In ML, DPUs have been used for control-plane scheduling (Conspirator~\cite{xiao2024conspirator}), offloading optimizer states (OptimusNIC~\cite{rebai2025optimusnic}), and accelerating object storage (OS2G~\cite{jin2025os2g}). Jain et al.~\cite{jain2021accelerating} explored DPU-based DNN acceleration through data augmentation and model validation offloading. However, utilizing SmartNICs for active staleness compensation in GNN training remains largely unexplored.
\section{Conclusion}
We introduce \sysname{}, a SmartNIC-assisted full-graph GNN training system designed to tackle the dominant inter-node communication bottleneck by moving lightweight embedding refinement into the data path. The core design places a linear-trend predictor on SmartNICs to extrapolate cached historical embeddings, applies boundary-node importance sampling to focus limited in-network compute on the most influential neighbors, and builds an asynchronous DPU–GPU pipeline and intermediate-result reuse to avoid new synchronization barriers. This mechanism complements existing full-graph training paradigms where redundant exchanges and irregular dependencies make communication the principal scalability constraint, and integrates cleanly with partition- and tensor-parallel frameworks for multi-GPU training. The feasibility of SmartNIC-side acceleration for GNN workloads motivates our in-network approach and its predictor design.

\section*{ACKNOWLEDGEMENTS}
This work is supported by Guangdong and Hong Kong Universities "1+1+1" Joint Research Collaboration Scheme (No. 2025A0505000012). 
Amelie Chi Zhou is the corresponding author.



\newpage
\section*{AI-Generated Content Acknowledgement}
We employed GPT-5 to refine the language of the entire manuscript. All generative outputs were manually reviewed and edited for technical accuracy.
\bibliographystyle{ieeetr}

\bibliography{myref}

\end{document}